\documentclass[10.5pt]{article}

\usepackage{paralist}
\usepackage{amsthm}

\usepackage{amsmath,amsfonts,bm}

\theoremstyle{plain}
\newtheorem{theorem}{Theorem}

\theoremstyle{definition}

\newtheorem{corollary}[theorem]{Corollary}

\theoremstyle{remark}
\newtheorem{remark}[theorem]{Remark}

\numberwithin{equation}{section}
\numberwithin{theorem}{section} 

\newenvironment{boxtheorem}[1][]{
  \refstepcounter{theorem}
  \begin{tcolorbox}[
    colback=blue!2!white,
    colframe=blue!50!black,
    enhanced,
    sharp corners,
    title={Theorem~\thetheorem%
      \if\relax\detokenize{#1}\relax
      \else:~#1%
      \fi}
  ]
}{
  \end{tcolorbox}
}

\newenvironment{boxdefinition}[1][]{
  \refstepcounter{theorem}
  \begin{tcolorbox}[
    colback=red!2!white,
    colframe=red!50!black,
    enhanced,
    sharp corners,
    title={Definition~\thetheorem%
      \if\relax\detokenize{#1}\relax
      \else:~#1%
      \fi}
  ]
}{
  \end{tcolorbox}
}

\newenvironment{boxinsight}[1][]{
  \begin{tcolorbox}[
    colback=red!2!white,
    colframe=red!50!black,
    enhanced,
    sharp corners,
    title={#1}
  ]
}{
  \end{tcolorbox}
}

\def\eqref#1{equation~\ref{#1}}

\def\1{\bm{1}}

\def\vx{{\bm{x}}}

\def\vz{{\bm{z}}}

\DeclareMathAlphabet{\mathsfit}{\encodingdefault}{\sfdefault}{m}{sl}
\SetMathAlphabet{\mathsfit}{bold}{\encodingdefault}{\sfdefault}{bx}{n}

\newcommand{\R}{\mathbb{R}}

\DeclareMathOperator*{\argmax}{arg\,max}

\usepackage{wrapfig}
\usepackage{mathpazo}
\usepackage[utf8]{inputenc}
\usepackage[T1]{fontenc}
\usepackage{verbatim}
\usepackage{calc}
\usepackage{mathrsfs}
\usepackage{mathtools}
\usepackage{algorithmic}
\usepackage{url}
\usepackage[ruled,vlined]{algorithm2e}  
\RestyleAlgo{ruled}   
\SetAlgoLined         
\usepackage{amsmath}
\usepackage{amssymb}
\usepackage{graphicx}
\usepackage{hyperref}
\usepackage[table]{xcolor} 
\usepackage{comment}
\usepackage{soul}
\usepackage{booktabs}
\usepackage{tikz}
\usetikzlibrary{matrix,arrows.meta,positioning}
\usepackage[most]{tcolorbox}

\definecolor{forestgreen}{rgb}{0.13, 0.55, 0.13}

\usepackage{tabularx}

\makeatletter

\newcommand{\lyxmathsym}[1]{\ifmmode\begingroup\def\b@ld{bold}
  \text{\ifx\math@version\b@ld\bfseries\fi#1}\endgroup\else#1\fi}

\usepackage{graphics}\usepackage{epsfig}

\usepackage{soul}

\usepackage{graphicx}
\usepackage{booktabs} 

\usepackage{here}

\usepackage{amsmath,amssymb,amsfonts,amsbsy,amsfonts,latexsym}
\usepackage{multirow}
\usepackage{makecell}
\usepackage[labelfont=bf,belowskip=0pt,aboveskip=5pt,tableposition=top]{caption}
\usepackage{xcolor}
\usepackage{colortbl}

\definecolor{colorA}{RGB}{189,201,225}
\definecolor{colorB}{RGB}{103,169,207}
\definecolor{colorC}{RGB}{ 28,144,153}
\definecolor{colorD}{RGB}{  1,108, 89}

\newcolumntype{R}{>{\columncolor{gray!40}}r}
\newcolumntype{L}{>{\columncolor{gray!40}}l}
\newcolumntype{C}{>{\columncolor{gray!40}}c}

\usepackage{tabularx,colortbl,xcolor}
\usepackage{multirow}
\usepackage[normalem]{ulem}
\useunder{\uline}{\ul}{}

\usepackage{enumitem}

\usepackage{xparse}

\DeclareGraphicsExtensions{.pdf,.png}

\SetKwInput{KwInput}{Input}

\usepackage{longtable}
\usepackage{pgfplots}
\usepackage{outlines}

\usepackage{caption}
\usepackage{subcaption}
\usepackage{graphbox} 

\tcbset{
    sharp corners,
    colback = white,
    before skip = 0.2cm,    
    after skip = 0.5cm      
}                           

\definecolor{main}{HTML}{4472C4}    
\definecolor{sub}{HTML}{EBF4FF}     

\newtcolorbox{boxA}{
    enhanced, breakable,
    boxrule = 0pt,
    colback = sub,
    borderline west = {2pt}{0pt}{main}, 
    borderline east = {2pt}{0pt}{main}, 
}

\usepackage{times}
\usepackage{textcomp}

\newcommand{\OURS}{{\textsc{Arbitrage}}\xspace}
\newcommand{\OURSORACLE}{{\textsc{Arbitrage Oracle}}\xspace}
\newcommand{\OURSROUTER}{{\textsc{Arbitrage Router}}\xspace}

\title{\OURS: Efficient Reasoning via Advantage-Aware Speculation}

\author{
Monishwaran Maheswaran\thanks{Equal Contribution}\hspace{0.5em}$^{1}$\enspace\enspace Rishabh Tiwari\footnotemark[1]\hspace{0.5em}$^{1}$\enspace\enspace Yuezhou Hu\footnotemark[1]\hspace{0.5em}$^{1}$\\  Kerem Dilmen$^{1}$ \enspace\enspace Coleman Hooper$^{1}$ \enspace\enspace Haocheng Xi$^{1}$ \enspace\enspace Nicholas Lee$^{1}$ \\
Mehrdad Farajtabar$^{2}$\enspace\enspace Michael W. Mahoney$^{1,3,4}$\enspace\enspace Kurt Keutzer$^{1}$\enspace\enspace Amir Gholami$^{1,3}$\vspace{3mm}
\\
{$^{1}$~UC Berkeley\qquad $^{2}$~Apple\qquad $^{3}$~ICSI\qquad $^{4}$~LBNL\vspace{3mm}}\\
}
\date{}  
\makeatother

\begin{document}

\maketitle

\begin{abstract}
Modern Large Language Models achieve impressive reasoning capabilities with long Chain of Thoughts, but they incur substantial computational cost during inference, and this motivates techniques to improve the performance-cost ratio. 
Among these techniques, Speculative Decoding accelerates inference by employing a fast but inaccurate draft model to auto-regressively propose tokens, which are then verified in parallel by a more capable target model. 
However, due to unnecessary rejections caused by token mismatches in semantically equivalent steps, traditional token-level Speculative Decoding struggles in reasoning tasks. 
Although recent works have shifted to step-level semantic verification, which improve efficiency by accepting or rejecting entire reasoning steps, existing step-level methods still regenerate many rejected steps with little improvement, wasting valuable target compute. 
To address this challenge, we propose \OURS, a novel step-level speculative generation framework that routes generation dynamically based on the \textit{relative advantage} between draft and target models. Instead of applying a fixed acceptance threshold, \OURS uses a lightweight router trained to predict when the target model is likely to produce a meaningfully better step. 
This routing approximates an ideal \OURSORACLE that always chooses the higher-quality step, achieving near-optimal efficiency–accuracy trade-offs. 
Across multiple mathematical reasoning benchmarks, \OURS consistently surpasses prior step-level SD baselines, reducing inference latency by up to $\sim2\times$ at matched accuracy.
Our code is available at \url{https://github.com/SqueezeAILab/Arbitrage}
\end{abstract}

\section{Introduction}
\begin{figure*}[t]
  \centering
  \includegraphics[width=\linewidth, trim=2 2 2 2, clip]{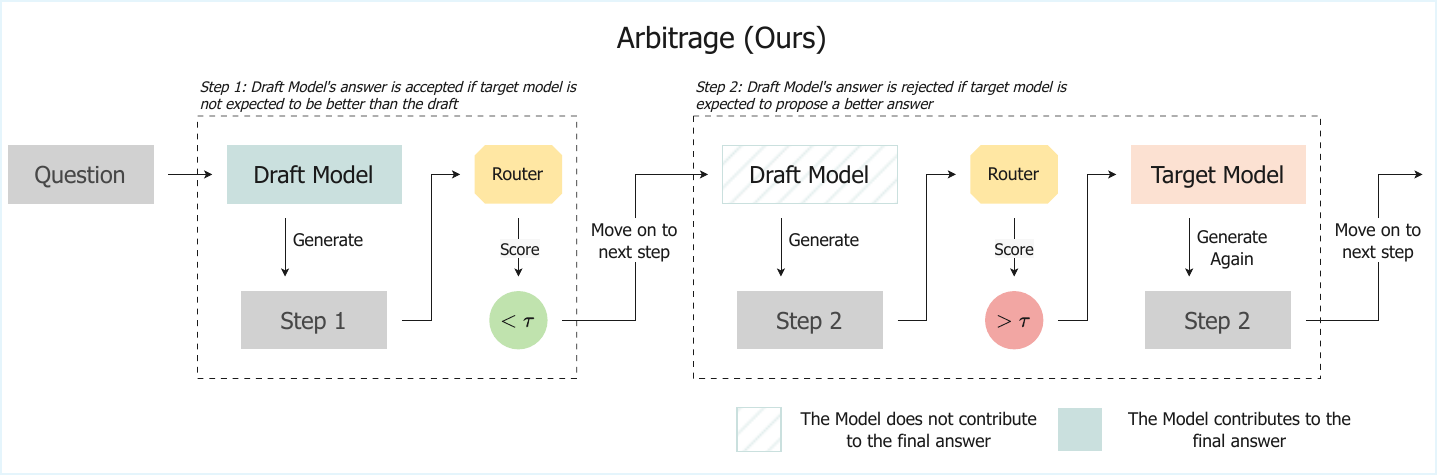}
  \vspace{2pt}
\caption{\textbf{\OURS overview.} At each reasoning step, the draft proposes a candidate. The router produces a score $\hat{y}$, which is the estimated probability that the target will outperform the draft on this step, and \emph{accepts} the draft if $\hat{y}\le\tau$, otherwise \emph{escalates} to the target to regenerate ($\hat{y}>\tau$). The selected step is appended to the context. The threshold $\tau$ governs the compute–quality trade-off.}
  \label{fig:arbitrage_intro}
\end{figure*}

Large Language Models (LLMs) have achieved remarkable progress across a wide range of domains, with mathematical reasoning standing out as a particularly transformative frontier \cite{yang2024qwen25mathtechnicalreportmathematical,shao2024deepseekmathpushinglimitsmathematical,qwq32b,yang2025qwen3technicalreport,deepseekai2025deepseekr1incentivizingreasoningcapability}. 
This rapid advancement has been fueled by the release of high-quality, reasoning-centric datasets \cite{zhou2025megamathpushinglimitsopen,guha2025openthoughtsdatarecipesreasoning,openr1-math-220k,gao2024omnimathuniversalolympiadlevel} as well as innovations in reinforcement learning algorithms tailored for complex reasoning tasks \cite{shao2024deepseekmathpushinglimitsmathematical,zheng2025groupsequencepolicyoptimization,yu2025dapoopensourcellmreinforcement}. As a result, state-of-the-art LLMs now rival, or even surpass, human performance on challenging mathematics and coding benchmarks, including international olympiad-level tasks \cite{he2024olympiadbenchchallengingbenchmarkpromoting,sun2025challengingboundariesreasoningolympiadlevel,zhu2025oibenchbenchmarkingstrongreasoning,jain2024livecodebenchholisticcontaminationfree}. Despite these impressive capabilities, the \emph{inference efficiency} of transformer-based LLMs remains a critical bottleneck. 
While training is typically compute-bound, auto-regressive decoding is fundamentally memory-bound: each token generation step relies on matrix-vector multiplications rather than the more hardware-efficient matrix-matrix operations. Consequently, the throughput of modern GPUs is limited not by raw compute power, but by global memory bandwidth, a phenomenon widely known as the “memory wall” \cite{gholami2024ai}. This constraint severely limits token generation speed and end-to-end latency, and the problem is exacerbated in long chain-of-thought reasoning, where solutions often span hundreds or thousands of tokens, leading to many decoding steps and substantial latency.
In this work, we focus on \emph{when} to spend expensive compute, rather than \emph{how} to make each forward pass cheaper. A more holistic strategy for accelerating end-to-end generation is Speculative Decoding (SD) \cite{leviathan2023fastinferencetransformersspeculative,chen2023acceleratinglargelanguagemodel}, which leverages parallelism by pairing a small, fast \textit{draft model} with a larger, more accurate \textit{target model}. The draft model auto-regressively generates multiple tokens in sequence, while the target model verifies them in parallel via a single forward pass; accepted tokens are appended to the output, and rejected tokens trigger fallback to the target model. In this paradigm, the central question is: \emph{when is it actually beneficial to switch from the draft to the target model?} Ideally, we should only invoke the target on those reasoning steps where it is expected to provide a meaningfully better continuation than the draft. 

Classical SD methods answer this question at the \emph{token level}, accepting or rejecting tokens based on exact agreement between draft and target distributions. 
Although effective in some settings, token-level SD suffers from low acceptance rates in complex reasoning tasks, particularly in chain-of-thought generation, where minor token-level discrepancies can lead to premature rejection of otherwise high-quality, semantically-equivalent reasoning steps \cite{liao2025rewardguidedspeculativedecodingefficient, pan2025specreason}, thus wasting computation. 
To address this challenge, recent work has shifted toward \textit{step-level} Speculative Decoding. 
Notably, Reward-guided Speculative Decoding (RSD) \cite{liao2025rewardguidedspeculativedecodingefficient} evaluates entire reasoning steps (e.g., delimited by \texttt{\textbackslash n\textbackslash n}) using a Process Reward Model (PRM) and accepts a draft step if its \emph{absolute} PRM score exceeds a global threshold, otherwise discarding it and regenerating with the target model. 
This thresholding rule depends only on whether the draft looks ``good enough'' in isolation, rather than on whether the target is actually expected to produce a \emph{better} step, making the routing decision fundamentally \emph{advantage-blind} and often triggering costly target regenerations that yield little or no quality gain.

In this work, we address this issue by proposing \OURS, a novel step-level speculative generation framework that dynamically routes between the draft and target models to minimize redundant computation while preserving, or even enhancing, output quality; see Figure~\ref{fig:arbitrage_intro} for an overview. 
Our key insight is to base the acceptance decision not just on the quality of the draft output (independent of the target model), but on the \textit{expected quality difference} between the draft and target models for a particular reasoning step. 
Conceptually, we define an \OURSORACLE that, for each step, compares the draft- and target-generated continuations and greedily selects the higher-quality one according to a given metric. 
Since running this oracle naively would require executing the target on every step, we instead train a lightweight \OURSROUTER that, given partial context, predicts how much a target-generated step is likely to outperform its draft counterpart. 
By anchoring decisions on the \emph{expected quality advantage}, \OURS avoids unnecessary target generations when the target model is expected to yield a similar-quality output as that of the draft model.

In summary, our main contributions are as follows.
\begin{itemize}
    \item 
    We systematically analyze the computational waste inherent in existing step-level SD methods, showing that a large fraction of regenerations fail to improve output quality (see Section \ref{sec:limitations}).
    \item 
    Based on our analysis, we propose \OURS, a step-level speculative generation framework that routes decisions based on the \emph{expected quality advantage} between the draft and target models, rather than an absolute quality threshold (see Section \ref{sec:method}).
    \item 
    We introduce a formal \OURSORACLE, which, at each reasoning step, compares the draft and target generated continuations and selects the higher-quality one according to a given metric. This defines a myopic (greedy) policy that is locally optimal for the per-step routing decision (see Section \ref{sec:oracle}).
    \item 
    We introduce a principled training pipeline for the \OURSROUTER, which is a practical lightweight model that can mimic \OURSORACLE during inference (see Section \ref{sec:router}).
    \item 
    We evaluate \OURS on different benchmarks and model settings, showing that it reduces inference latency by up to $\sim2\times$ over baselines for a given accuracy target (see Sections \ref{sec:algo_results} and \ref{sec:speedup}). 
\end{itemize}
\section{Related Work}

\subsection{Chain-of-Thought Reasoning}  
Chain-of-Thought (CoT) reasoning has been shown to improve model performance without requiring additional training. 
The simplest approach uses prompt engineering, e.g., instructing models to ``think step by step'' to elicit multi-step reasoning in instruction-following LLMs \cite{wei2023chainofthoughtpromptingelicitsreasoning}. More advanced methods employ search algorithms \cite{snell2024scalingllmtesttimecompute,zhang2024accessinggpt4levelmathematical,muennighoff2025s1simpletesttimescaling} or reward models \cite{lightman2023letsverifystepstep,xu2025genarmrewardguidedgeneration,zhao2025genprmscalingtesttimecompute} to guide generation toward higher-quality reasoning paths. To directly enhance inherent reasoning capabilities of models, recent work applies Supervised Fine-Tuning (SFT) and Reinforcement Learning (RL)\cite{pang2025boltbootstraplongchainofthought,yeo2025demystifyinglongchainofthoughtreasoning,deepseekai2025deepseekr1incentivizingreasoningcapability,shao2024deepseekmathpushinglimitsmathematical,khatri2025art,yang2024qwen25mathtechnicalreportmathematical,guha2025openthoughtsdatarecipesreasoning}. These training-based approaches significantly boost model capacity, enabling smaller models to match or even surpass larger counterparts in reasoning performance.

\subsection{Speculative Decoding for Reasoning}  
Speculative Decoding (SD) \cite{leviathan2023fastinferencetransformersspeculative} accelerates autoregressive generation by using a small, fast \textit{draft model} to propose tokens that are then verified in parallel by a larger, more accurate \textit{target model}. 
Traditional SD methods operate at the token level \cite{kim2023speculativedecodingbiglittle,liu2024onlinespeculativedecoding,zhou2024distillspecimprovingspeculativedecoding,tiwari2025quantspecselfspeculativedecodinghierarchical,ankner2024hydrasequentiallydependentdraftheads,li2025eagle3scalinginferenceacceleration,cai2024medusasimplellminference,hu2025adaspecselectiveknowledgedistillation}. 
However, token-level alignment results in premature repeated rejections in reasoning tasks where semantic level equivalence is enough. 
To address this, recent work shifts verification from tokens to entire reasoning steps.
This is typically delimited by structural markers such as double newlines (\texttt{\textbackslash n\textbackslash n}) in CoT outputs \cite{liao2025rewardguidedspeculativedecodingefficient} or phrases like ``wait'' or ``let me rethink'' in reasoning chains \cite{yang2025speculativethinkingenhancingsmallmodel, pan2025specreason}. 
This step-level approach avoids discarding an entire high-quality reasoning trajectory due to a single rejected intermediate token, thereby improving both efficiency and robustness. 

\subsection{Reward Models for Reasoning Guidance}  
Reward Models (RMs) play a central role in evaluating and guiding speculative generation by scoring output quality. Outcome-supervised Reward Models (ORMs) assign rewards based solely on the correctness of the final answer \cite{zhu2024starlingb,liu2024skyworkrewardbagtricksreward}. In contrast, Process-supervised Reward Models (PRMs) ~\cite{he_2024_16998085,wang2024math,zhu2025retrievalaugmentedprocessrewardmodel,uesato2022solvingmathwordproblems,lightman2023letsverifystepstep} provide fine-grained, per-step quality assessments of the reasoning process itself. By evaluating intermediate logical steps, PRMs enable more nuanced feedback and are particularly well-suited for step-level SD frameworks. In this work, we focus on CoT-style reasoning, where PRMs provide reliable step-level signals.

\section{Preliminaries}
\label{sec:preliminaries}
We begin by formalizing reward-based speculative generation through the lens of a simple guiding principle: we should only route from the draft model to the target model when doing so is expected to \emph{help}—that is, when the target is likely to provide a strictly better reasoning step than the draft under the same evaluation signal. Concretely, we consider step-level speculative generation, in which a Process Reward Model (PRM) scores entire candidate reasoning steps and these scores are used to decide whether to accept the draft or regenerate with the target. Prior work, illustrated by RSD \cite{liao2025rewardguidedspeculativedecodingefficient}, fits into this framework but makes routing decisions based solely on the draft step’s PRM score compared against a fixed global threshold. Whenever the draft score falls below this threshold, the system invokes the target model, regardless of whether the target is actually likely to improve the step. This absolute-score rule often triggers costly target regenerations that yield little or no quality gain. In Section \ref{sec:overview_rsd} , we review this baseline in detail; in Section \ref{sec:limitations}, we analyze its computational and behavioral limitations, quantifying the resulting wasted compute. These observations motivate a \emph{relative}, advantage-based routing strategy (discussed in Section \ref{sec:method}), in which the draft-to-target switch is explicitly tied to the target model’s \emph{expected advantage} over the draft.

\subsection{Overview of RSD}
\label{sec:overview_rsd}
Like classical SD, RSD employs a lightweight draft model in conjunction with a more powerful target model during text generation. A key distinguishing feature of RSD is its integration of a PRM that provides step-level supervision. By operating at a larger granularity through the use of reward signals from a PRM, RSD can adaptively accept high-value draft outputs, rather than rejecting them due to token-level mismatch. At each generation step, the draft model proposes a candidate reasoning step, which is then scored by the PRM. Based on this score, the system decides whether to accept the draft output or fallback to the target model for regenerating the step. However, unlike classical SD, because the PRM evaluates the entire reasoning step holistically rather than token-by-token, it avoids the computational overhead of fine-grained token-wise verification, thereby maintaining efficiency. 

Formally, let the input context be a sequence of tokens $\vx = [x_1, x_2, \dots, x_t]$. The draft model $\theta_{\text{draft}}$ or target model $\theta_{\text{target}}$ can autoregressively generate a candidate reasoning step $\vz = \{z_i\}_{i=1}^{\gamma}$ until it produces a designated separator token (e.g., \texttt{\textbackslash n\textbackslash n}), which marks the end of the step. 
This generation process is defined as:
\[
z_{i} \sim q_{\theta_{\text{model}}}(z_i \mid \vx, \vz_{<i}), \quad \text{for } i = 1, \dots, \gamma,
\]
where $\vz_{<i} = [z_1, \dots, z_{i-1}]$ and $z_\gamma$ is the separator token. 
Here, $\theta_{\text{model}}$ can be either $\theta_{\text{draft}}$ or $\theta_{\text{target}}$. 
We use $\vz_d$ to denote the draft generated step and $\vz_t$ be the target generated step. 
We first generate $\vz_d$, which is then evaluated by the PRM, defined as $h_{\theta_{\text{PRM}}}:\mathcal{X}\times\mathcal{Z}\to\R$. 
The PRM model outputs a scalar reward: 
\[
s_d \;=\; h_{\theta_{\text{PRM}}}(\vx,\vz_d).
\]
This score is used to decide whether to accept the speculation or reject it. The binary \emph{accept} indicator is
\[
A(\vx,\vz_d)\;=\;\mathbb{I}\{s_d>\tau\},\qquad \tau\in\R.
\]
That is, the step $\vz_d$ is accepted if $s_d > \tau$ (where $\tau$ is a fixed threshold); otherwise, $\vz_d$ is discarded and $\vz_t$ is generated by the target model. For a question that requires $n$ reasoning steps, the empirical acceptance rate is $\alpha = \frac{1}{n} \sum_{i=1}^n A_i$, where $A_i$ is the accept indicator of step $i$. Crucially, $\alpha$ is \textit{inversely} related to $\tau$: a higher threshold yields \textit{lower} acceptance (fewer steps accepted), trading speed for quality. By tuning $\tau$, RSD controls the efficiency–accuracy trade-off. This PRM-guided mechanism dynamically balances computational cost and output quality. 

\subsection{Limitations of Current Approach} 
\begin{figure*}[t]
  \centering
  \includegraphics[width=\linewidth, trim=10 40 10 20, clip]{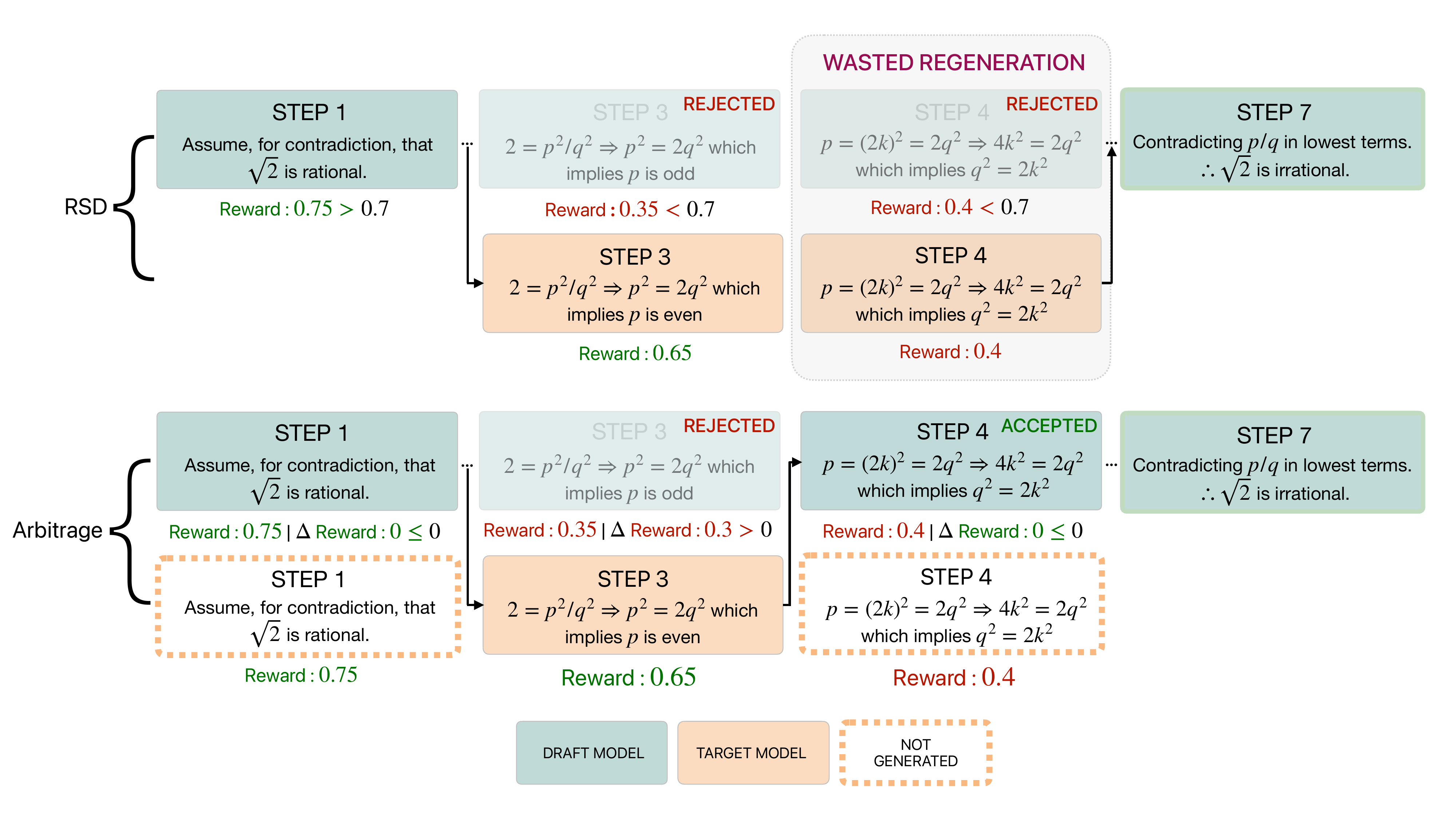}
  \vspace{-2pt}
  \caption{\textbf{\OURS vs.\ baseline step-level SD approaches.}
  Comparison of Reward-guided Speculative Decoding (RSD, top, which we use as a baseline) and our \OURS algorithm (bottom).
  RSD accepts or rejects draft steps using an absolute PRM reward threshold: when the PRM score of a draft-generated step falls below this threshold, the step is discarded and the target model is invoked to regenerate it.
  This absolute criterion can trigger unnecessary target regenerations (e.g., Step 4), where the target does not significantly improve the quality of the step.
  \OURS instead estimates the expected quality gain from \emph{escalating} a step, i.e., invoking the larger target model to regenerate the step rather than keeping the draft step.
  It only calls the target when this predicted gain is positive, thereby avoiding wasted target calls (e.g., Steps 1 and 4).}
  \label{fig:arbitrage_method}
\end{figure*}

\begin{figure}[t]
    \centering
    \includegraphics[width=0.5\linewidth]{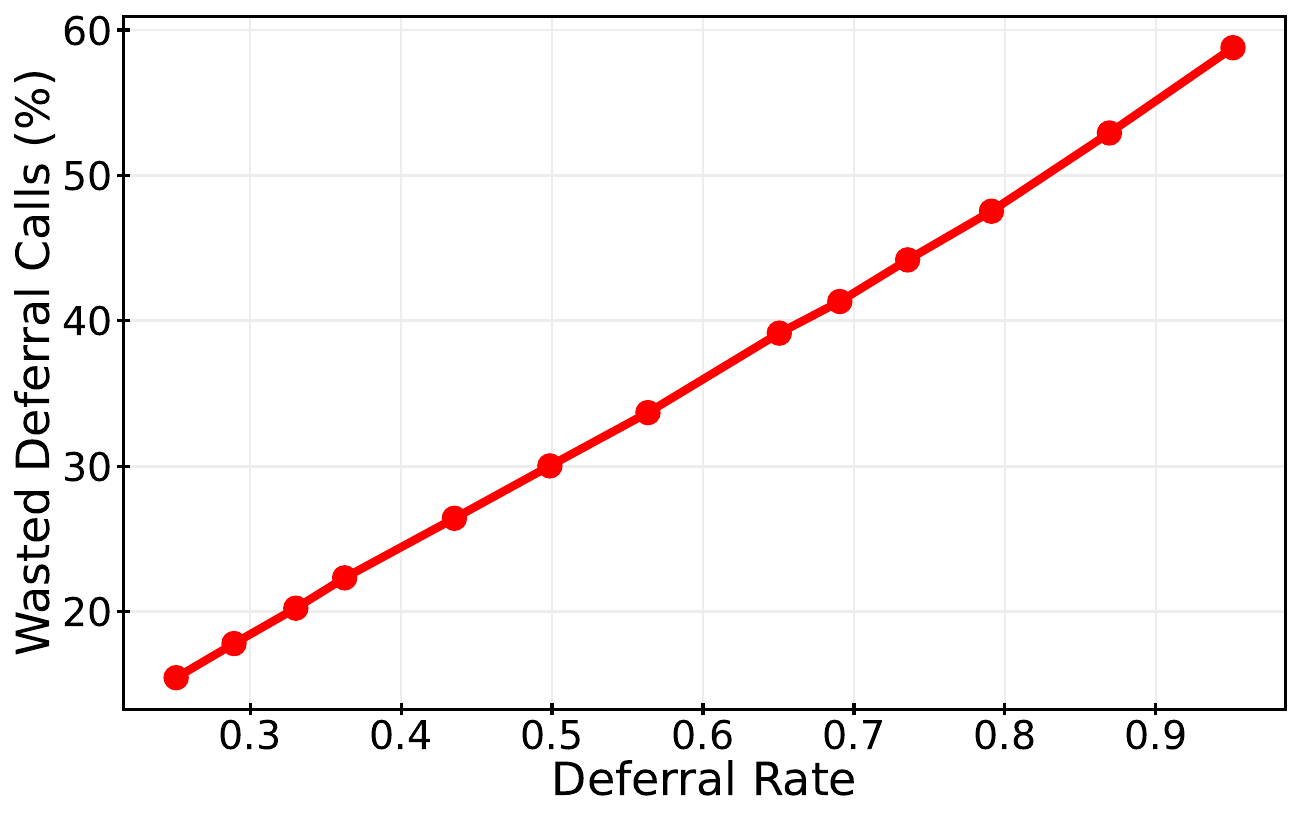}
    \caption{\textbf{Wasted target calls vs.\ deferral rate.} 
    In reward-based step-level speculation, the \emph{deferral rate} (x-axis) is the fraction of reasoning steps that are escalated to the target model. 
    The \emph{wasted deferral rate} (y-axis) is the percentage of those escalations where the target’s step is \emph{no better} than the draft (equal or lower PRM score), relative to total number of steps. Wasted compute increases steadily with deferral rate, indicating many unnecessary target invocations under absolute-score rejection.}
    \label{fig:rsd-limitation}
\end{figure}

\label{sec:limitations}
While \cite{liao2025rewardguidedspeculativedecodingefficient} demonstrates the effectiveness of reward-based speculative generation, we identify a fundamental inefficiency in its rejection mechanism. When the PRM rejects a draft step based on its reward score, regenerating that step with the target model often provides little or no improvement in quality. To build intuition for how this mechanism behaves, Figure~\ref{fig:arbitrage_method} presents a schematic example of step-level decisions under RSD and our method. Crucially, this regeneration incurs the full computational cost, despite yielding negligible quality gains, and this leads to substantial wasted compute. To illustrate this, let $s_t = h_{\theta_{\text{PRM}}}(\vx,\vz_t)$ be the corresponding reward for the target generated step if $\vz_d$ is rejected ($A(\vx,\vz_d))=0$). The regeneration incurs cost $C_t$, while the previous draft generation cost $C_d$ is paid regardless.

\noindent\textbf{When does rejection help?} Quality improves only if the target step outperforms the draft step under the same PRM:
\[
\text{benefit} \;=\; \mathbb{I}\{A=0\}\,\mathbb{I}\{s_t>s_d\}.
\]
Otherwise, the expensive regeneration is \emph{wasted}. We formalize the expected wasted computation of current approaches~as:
\[
\mathcal{W}_{\text{RSD}}
\;\triangleq\;
\mathbb{E}\big[\, C_t \cdot \mathbb{I}\{A=0\}\cdot \mathbb{I}\{s_t\le s_d\}\,\big],
\]
which grows both with the rejection rate $\Pr(s_d\le\tau)$ and with the failure probability $\Pr(s_t\le s_d\mid s_d\le\tau)$. Observe that this inefficiency comes from two key factors.
\begin{enumerate}
    \item
    On some difficult steps, the PRM assigns a low score to a draft that is essentially correct, causing a deferral to the target. The target then produces a step of comparable quality, yielding no gain, but incurring the  full computational cost.
    \item
    For the hardest instances, the draft and target may exhibit the same logical gaps or reasoning errors, so regeneration is unlikely to improve quality.
\end{enumerate}

We empirically validate this phenomenon. 
Figure~\ref{fig:rsd-limitation} summarizes the wasted compute due to useless deferrals to the target model for regeneration. 
We observe that this wasted compute w.r.t. total cost increases with higher deferral rates. 
For example, at a deferral rate of 70\%, approximately 40\% total steps were regenerated with the target model without any gain in output quality. 
See Appendix~\ref{app:limitations} for further analysis.
\section{Method}
\label{sec:method}
In this section, we formalize our proposed framework: \OURS, which aims to minimize redundant computation in step level speculative decoding methods. 
Building on the inefficiencies we identified in Section~\ref{sec:limitations}, our goal is to design a routing policy that dynamically decides at each reasoning step whether invoking the target model will meaningfully improve the step quality. 
We begin in Section~\ref{sec:oracle} by defining the \OURSORACLE, which establishes the theoretical upper bound for routing efficiency by comparing the counterfactual rewards of draft and target generations. We then in Section~\ref{sec:router} introduce the \OURSROUTER, a lightweight predictive model that approximates this oracle using only the draft’s context, thereby enabling near-optimal routing decisions, without executing the expensive target model.

\subsection{\OURSORACLE: The Ideal Routing Strategy}
\label{sec:oracle}
Building on the identified limitations of absolute-score-based rejection, we formalize the \OURSORACLE. Unlike conventional approaches that accept or reject draft steps based solely on the draft’s PRM score, $s_d$, the \OURSORACLE makes routing decisions by comparing the draft and target models’ counterfactual PRM scores for the \emph{same} reasoning step. 
Specifically, given a draft step $\vz_d$ and its target-generated counterpart $\vz_t$ for context $\vx$, the oracle selects the sequence with the higher PRM score:
\begin{equation}
\vz^* = \argmax_{\vz \in \{\vz_d, \vz_t\}} h_{\theta_{\text{PRM}}}(\vx, \vz).
\end{equation}  
This strategy establishes an upper bound for routing performance: it achieves greedy-maximal output quality for any given acceptance rate by always selecting the highest-quality step available.

To make this principle concrete while controlling the acceptance rate, we introduce two simple notions. 
First, we define the \emph{advantage} of the target over the draft on a given step:
\begin{boxdefinition}[Step-level advantage]
\label{def:step_level_advantage}
Given a context $\vx$, a draft step $\vz_d$ with PRM score $s_d = h_{\theta_{\text{PRM}}}(\vx,\vz_d)$ and a target step $\vz_t$ with score $s_t = h_{\theta_{\text{PRM}}}(\vx,\vz_t)$, the \emph{advantage} of the target over the draft on this step is
\begin{equation}
\Delta = s_t - s_d.
\end{equation}
A positive $\Delta$ indicates that the target is strictly preferred to the draft under the PRM.
\end{boxdefinition}
This mirrors the standard “advantage” terminology from reinforcement learning, and a positive $\Delta$ indicates that the target is strictly preferred to the draft under the PRM.
Second, we define an \emph{escalation decision}, which encodes the routing choice of whether to accept the draft step or escalate to the target model:

\begin{boxdefinition}[Escalation decision and realized score]
\label{def:escalation_decision_and_realized_score}
An \emph{escalation decision} is a binary variable $a \in \{0,1\}$, where $a = 1$ means we \emph{escalate} to the target model (regenerate the step with the target) and $a = 0$ means we \emph{do not escalate} and accept the draft step. The corresponding realized PRM score under decision $a$ is
\begin{equation}
S(a) \;=\; (1-a)\,s_d \;+\; a\,s_t \;=\; s_d \;+\; a\,\Delta.
\end{equation}
\end{boxdefinition}

A routing \emph{policy} $\pi$ maps available information to $a\in\{0,1\}$. 
When controlling the escalation rate via a scalar threshold $\tau\in\mathbb{R}$, the oracle uses the thresholding policy
\begin{equation}
a^\star_\tau \;=\; \mathbb{I}\{\Delta>\tau\}.
\end{equation}
The threshold $\tau$ regulates the trade-off between accuracy and how often we escalate to the target, and thus directly controls the computational cost. 
Appendix~\ref{sec:proof} proves that, under a greedy per-step objective, this thresholding rule is optimal: among all policies with the same escalation rate, it maximizes the expected PRM score $S(a)$.

While the \OURSORACLE establishes the upper bound for routing performance, it remains computationally infeasible in practice. 
Evaluating $s_t$ requires executing the target model to obtain $\vz_t$, precisely the costly operation that SD aims to avoid. Consequently, our central challenge is to approximate the oracle’s routing decision \emph{without} invoking the target model itself. Formally, this reduces to estimating the expected advantage, conditioned on the observed context and draft step:
\begin{quote}
\textit{Can we estimate how much better the target would score than the draft on the current step, without actually running the target model?}
\end{quote}
A practical solution to this question would enable near-oracle routing decisions while preserving the speedups of SD. To this end, the next subsection introduces the \OURSROUTER, a lightweight predictive model trained to estimate this conditional advantage from partial context, thus approximating oracle-level routing at minimal additional cost.

\subsection{\OURSROUTER: The Practical Routing Strategy}
\label{sec:router}
We address the above challenge of the impracticality of computing the  \OURSORACLE with a lightweight \OURSROUTER that, given only the draft-side context and step: $(\vx,\vz_d)$, predicts whether escalating will improve quality. Concretely, the router outputs a score $\hat{y}=h_{\theta_{\text{router}}}(\vx,\vz_d)$ interpreted as the likelihood that the target outperforms the draft. For routing during inference, we use a threshold $\tau$ (reject iff $\hat{y}>\tau$), enabling advantage-aware decisions without executing the target during routing. The \OURSROUTER is trained offline using data labeled by the oracle (see Section~\ref{sec:oracle}) to learn its routing decisions. 

\begin{boxinsight}[Oracle-free routing objective]
\label{def:routing_objective}
Given context $\vx$ and draft step $\vz_d$, predict the oracle advantage
\[
\Delta(\vx,\vz_d) = s_t - s_d,
\]
or equivalently its sign $y = \mathbb{I}[\Delta>0]$, using only draft-side
information, without executing the target model.
\end{boxinsight}

\subsubsection*{Dataset Construction}
\label{sec:dataset_construction}
Using the step-level advantage $\Delta$ (Definition~\ref{def:step_level_advantage}) and the oracle label $y$, we construct a step-level dataset as follows. The oracle label for escalation is
\[
y = \mathbb{I}[\Delta > 0],
\]
indicating whether invoking the target yields higher quality. For each context $x$, we decode both the draft and target models from the same prefix to obtain paired steps $(z_d, z_t)$, and compute their PRM scores $s_d$ and $s_t$ with the fixed PRM $h_{\theta_{\text{PRM}}}$, as in Definition~\ref{def:step_level_advantage}. The resulting tuples
\[
(x, z_d, z_t, s_d, s_t, \Delta, y)
\]
form the supervised training data for the router. To reduce variance in the oracle signal, we optionally draw multiple target samples per context and average their PRM scores to obtain $\bar{s}_t$ and the corresponding averaged advantage $\bar{\Delta}$; the oracle label $y$ is then computed from this averaged advantage.

\subsubsection*{Data Preprocessing}
\label{sec:preprocess}
We sample $30\text{K}$ questions via stratified sampling from the NuminaMath-CoT dataset~\cite{numina_math_datasets}, to be used as a seed dataset for our fine-tuning. 
Because the majority of draft steps are acceptable, the oracle labels $y$ exhibit strong class imbalance ($y = 0$ dominates $y = 1$). 
For example, with the Llama family configuration, the label distribution was $62.27\%$ for $y = 0$ and $37.73\%$ for $y = 1$. For better performance, we add annotations to each history step indicating the invoked models. We apply random downsampling of the majority class to balance the training set and prevent bias toward over-accepting draft steps.
We also normalize sequence lengths and enforce a consistent step separator token (\texttt{\textbackslash n\textbackslash n}) to align PRM scoring and router inputs.

\subsubsection*{Training \OURSROUTER}
\label{sec:training}
The router is initialized from a compact PRM checkpoint (reasoning-aware, step-level pretraining) to provide an inductive bias for evaluating intermediate reasoning quality.
We finetune the model with the AdamW optimizer~\cite{loshchilov2019decoupledweightdecayregularization}, using linear warmup and decay. The classification head is applied on the final token embedding, and the model is optimized via standard cross-entropy loss. The training hyperparameters are presented in the Table~\ref{tab:hyperparams} in Appendix~\ref{app:hparams}.

\subsubsection*{Routing Quality Evaluation.}
\label{sec:router_quality_evaluation}

Evaluating the quality of the router during training is non-trivial. Unlike standard classifiers, we cannot afford to perform full end-to-end evaluation of the entire SD pipeline across a range of thresholds~$\tau$, as each such evaluation would require executing both the draft and target models, which would be prohibitively expensive. Consequently, we need a reliable \emph{proxy metric} that reflects how well the router aligns with the oracle’s routing decisions, without repeatedly running the full algorithm. Traditional metrics such as accuracy, precision, or recall at a fixed threshold can serve as a simple evaluation metric to judge the quality of the router. Since $y=0$ and $y=1$ leads to generations from different models, we observe that:
\begin{enumerate}
    \item Accuracy of class $y = 0$ (accept draft) ensures high acceptance rates and fast generation.
    \item Accuracy of class $y = 1$ (reject draft) ensures that challenging steps—where the draft model is likely to fail—are delegated to the target model, preserving output quality and overall accuracy.
\end{enumerate}
However, these traditional metrics are heavily dependent on the chosen $\tau$, and they do not capture how the router’s predictions correlate with the true advantage signal that drives the oracle. Thus, they provide limited insight: a router may achieve high accuracy at a specific threshold, and yet still mis-rank steps by their expected improvement, leading to suboptimal routing behavior when $\tau$ changes. 
To solve this, and to use a fixed, global metric to evaluate overall quality, we adopt \textbf{Spearman’s rank correlation}~\cite{spearmanpaper} between the predicted routing probabilities~$\hat{y}$ and the oracle advantage scores~$\Delta$:
\begin{equation}
    \rho = \mathrm{corr}_S(\hat{y}, \Delta).
    \label{eq:spearman}
\end{equation}
This rank-based correlation serves as a threshold-invariant proxy for router quality, measuring how well the predicted ordering of steps matches the oracle’s ground-truth ordering of advantages. A high~$\rho$ indicates that the router consistently assigns higher probabilities to steps where escalation truly helps, implying robust performance across a wide range of thresholds, even without explicitly re-running the full decoding algorithm.

\subsubsection*{\OURS Speculative Generation.}
\label{sec:algorithm}
Figure~\ref{fig:arbitrage_intro} shows the overall \OURS speculative generation workflow during inference with \OURSROUTER. The process is formulated as a step-level speculation: (i) the draft model proposes a candidate step; and (ii) the router decides whether to accept it or escalate to the target model for regeneration. Algorithm~\ref{alg:arbitrage} in Appendix~\ref{app:algorithm} summarizes the complete decoding procedure. The router introduces a small overhead of one forward pass per step, while offering fine-grained control over the compute–quality trade-off through the threshold~$\tau$.
\section{Evaluation}
In this section, we evaluate the performance of our framework, \OURS, across different model configurations and reasoning benchmarks. We first analyze accuracy under varying \emph{acceptance rates} (as defined in Section~\ref{sec:overview_rsd}), which control the balance between using the draft and target models. We then present end-to-end latency results, showing that \OURS achieves substantial speedups by performing fewer, more effective escalations to the target model.

\subsection{Setup}
All experiments were conducted on NVIDIA A6000 GPUs using SGLang as the inference backend, each model occupying a dedicated GPU. For speedup measurements, we fix the batch size to 1. We evaluate \OURS on two representative model families, LLaMA3~\cite{grattafiori2024llama} and Qwen2.5-Math~\cite{yang2024qwen2}, under two practical draft–target regimes:
\begin{enumerate}
    \item A large target with a smaller draft from the same family.
    \item A large target with a weight-quantized version of itself as the draft.
\end{enumerate}
We use llama.cpp for quantizing Qwen models and GPTQ~\cite{frantar2023gptqaccurateposttrainingquantization} for quantizing Llama models. All the models chosen are instruction finetuned versions from their respective family. For the PRM, unless stated otherwise, we use Skywork-o1-Open-PRM (1.5B)~\cite{skywork_o1_open_series_2025}, and we fine-tune 
Qwen2.5-Math-1.5B-Instruct-PRM-0.2~\cite{uesato2022solving} to obtain an \OURSROUTER tailored to each target–draft pair. For brevity, from here on we use notations like \OURS (4bit-7B/7B) to denote \OURS based inference using a 4bit weight quantized 7B draft model, a bf16 7B target model and 1.5B router.

\subsection{Results}
\label{sec:algo_results}
\begin{figure*}[t]
  \centering

  \begin{minipage}[t]{0.33\textwidth}
    \centering
    \includegraphics[width=\linewidth]{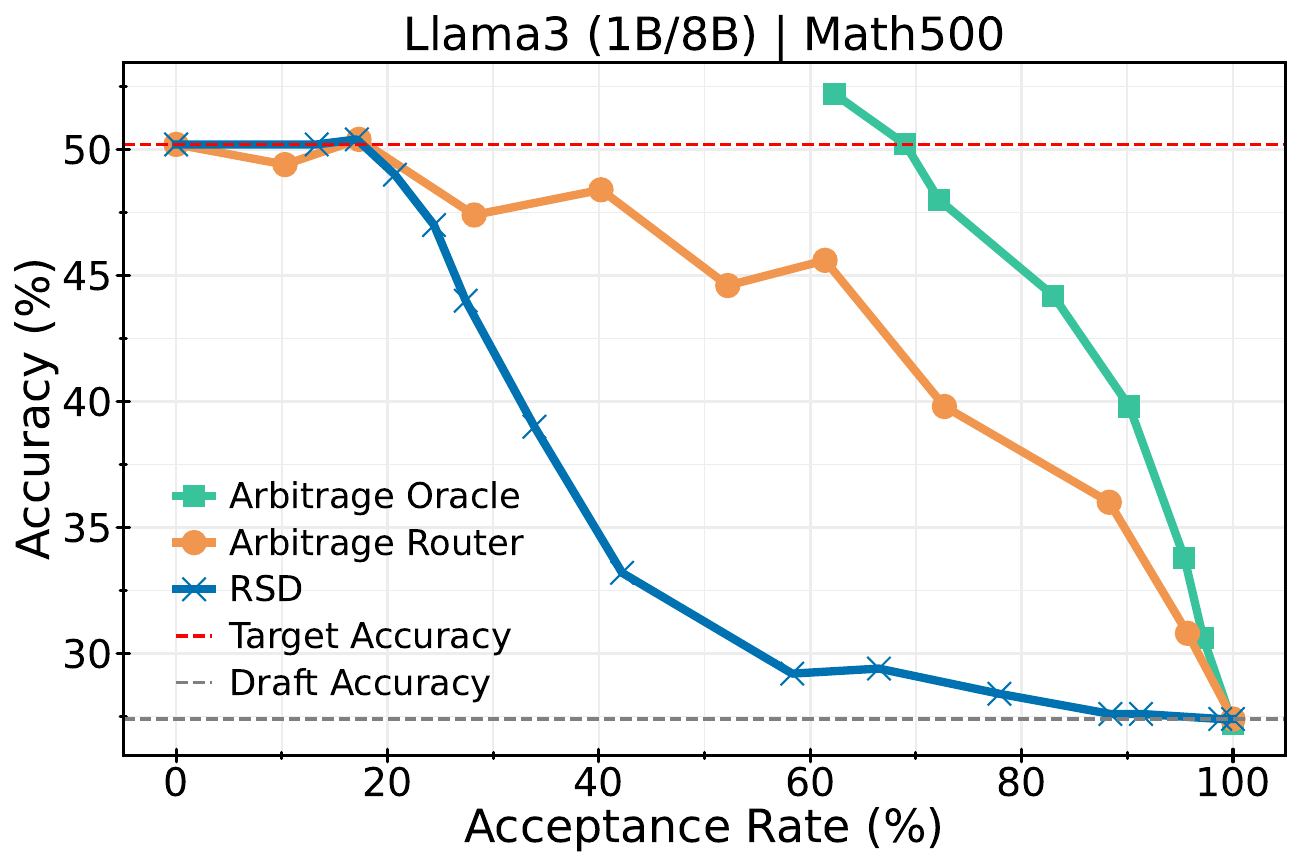}
    \par\vspace{0.25em}
    \includegraphics[width=\linewidth]{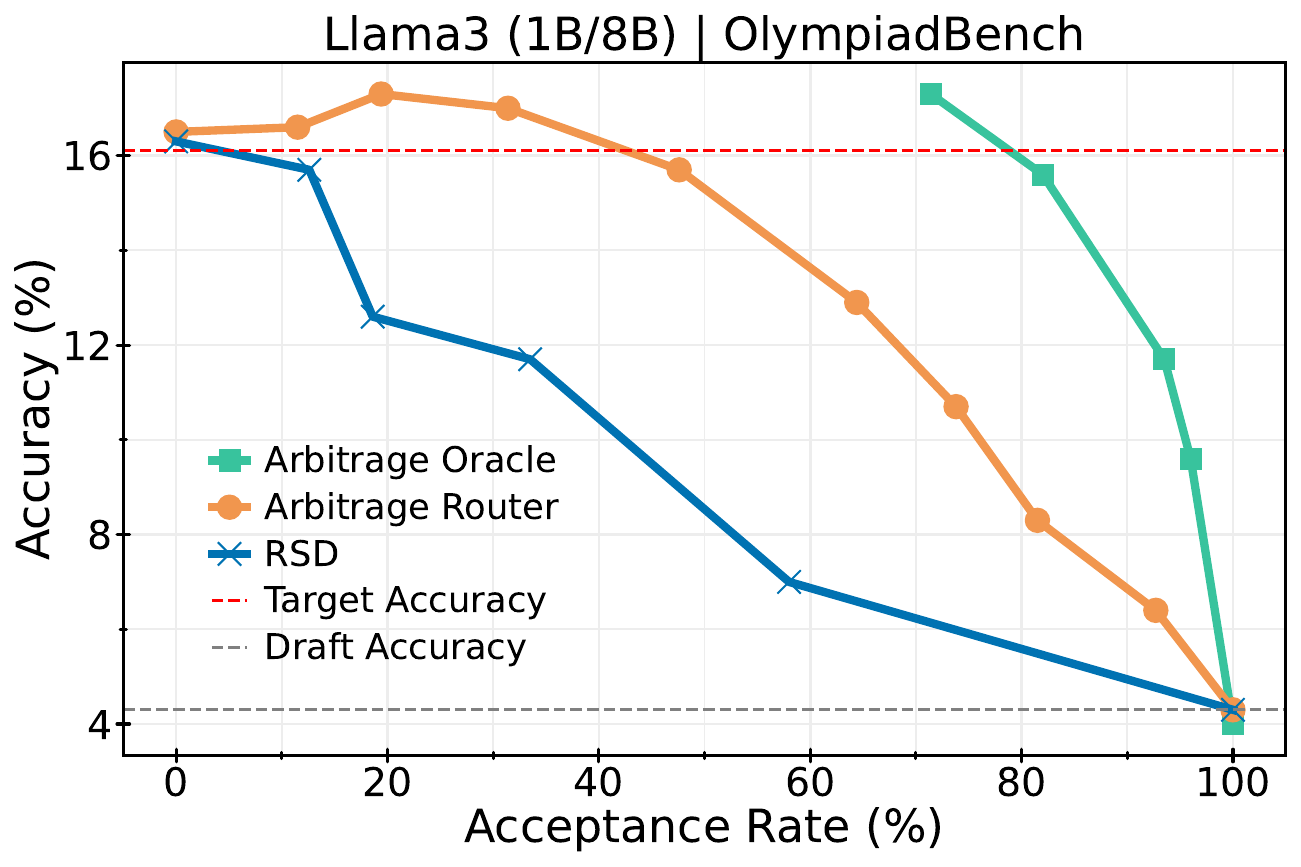}
    \par\vspace{-1pt}\small(a) \
  \end{minipage}\hfill
  %
  \begin{minipage}[t]{0.33\textwidth}
    \centering
    \includegraphics[width=\linewidth]{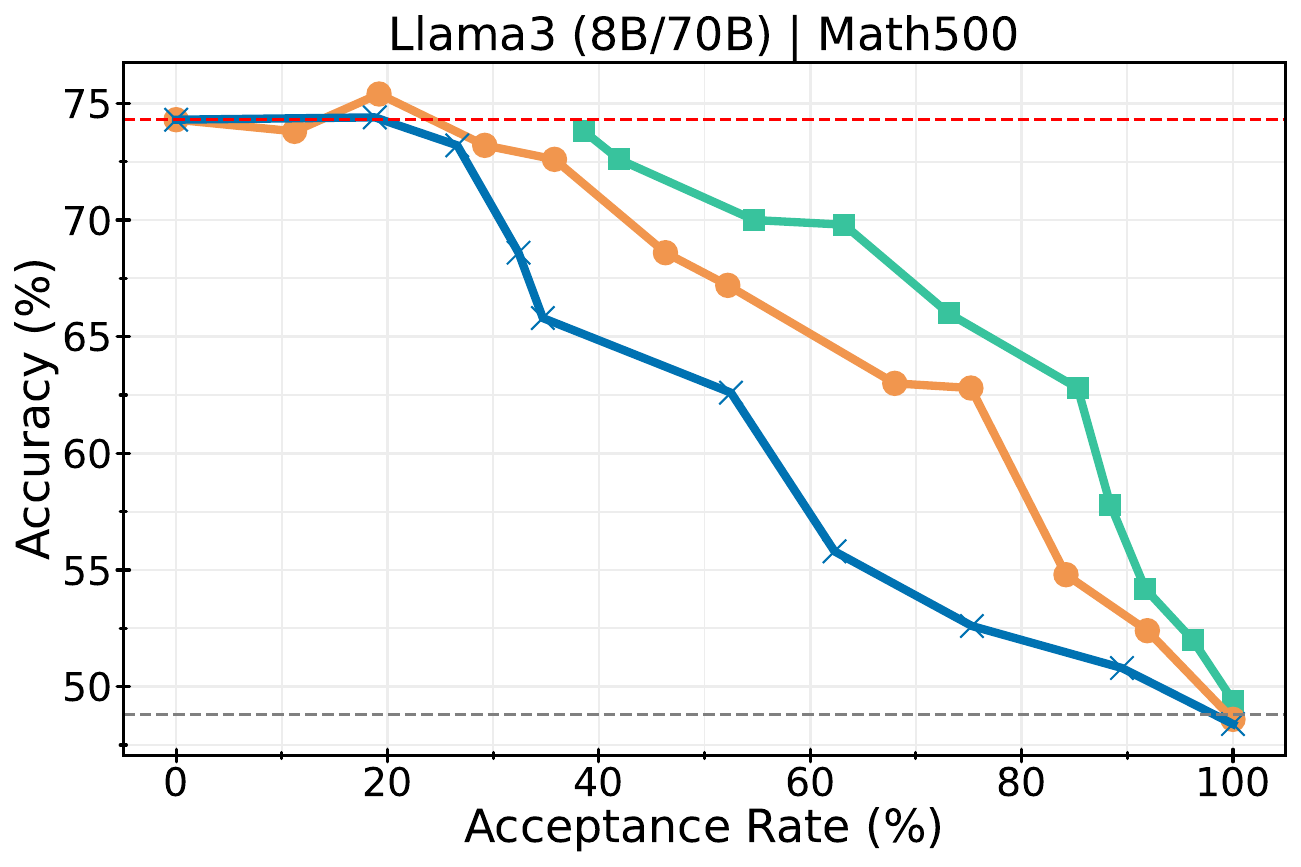}
    \par\vspace{0.25em}
    \includegraphics[width=\linewidth]{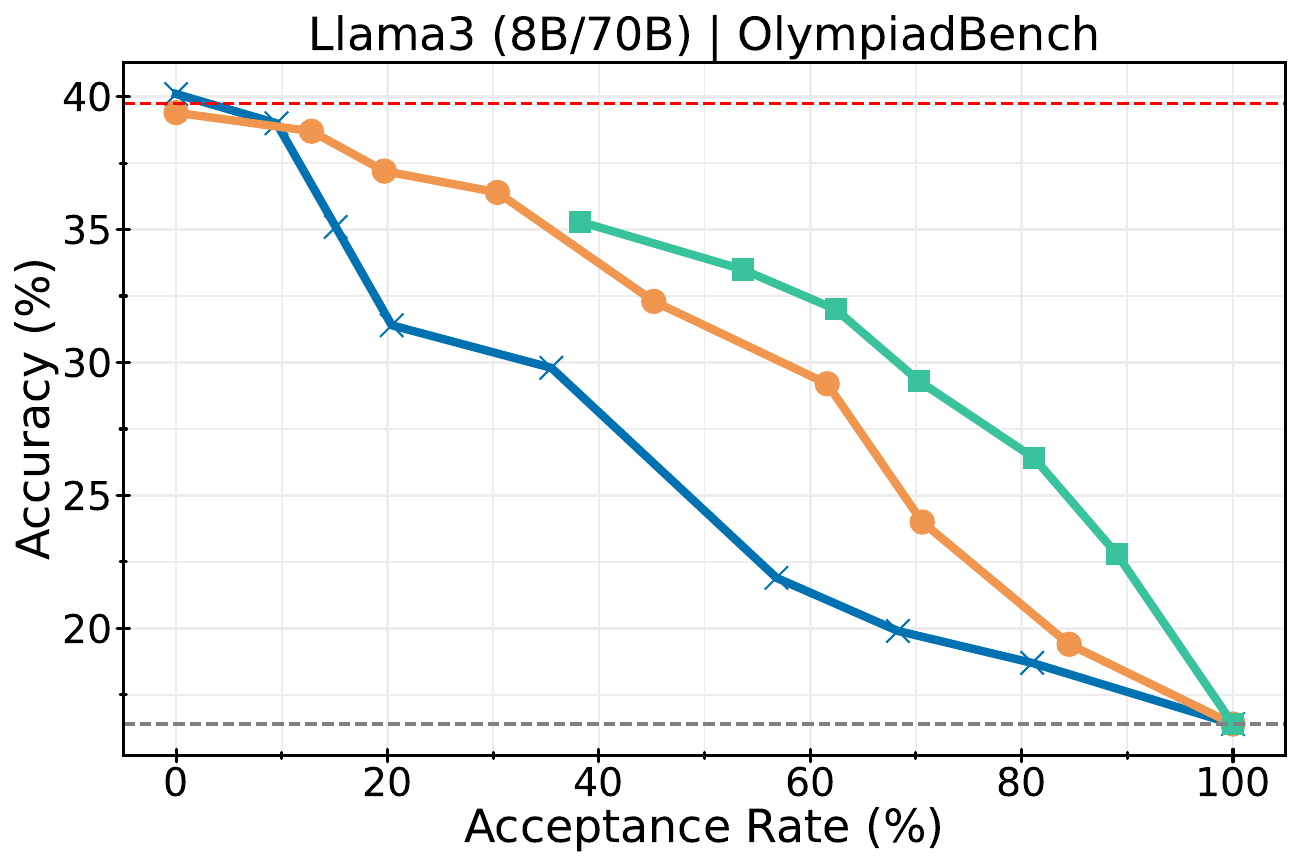}
    \par\vspace{-1pt}\small(b) \
  \end{minipage}\hfill
  %
  \begin{minipage}[t]{0.33\textwidth}
    \centering
    \includegraphics[width=\linewidth]{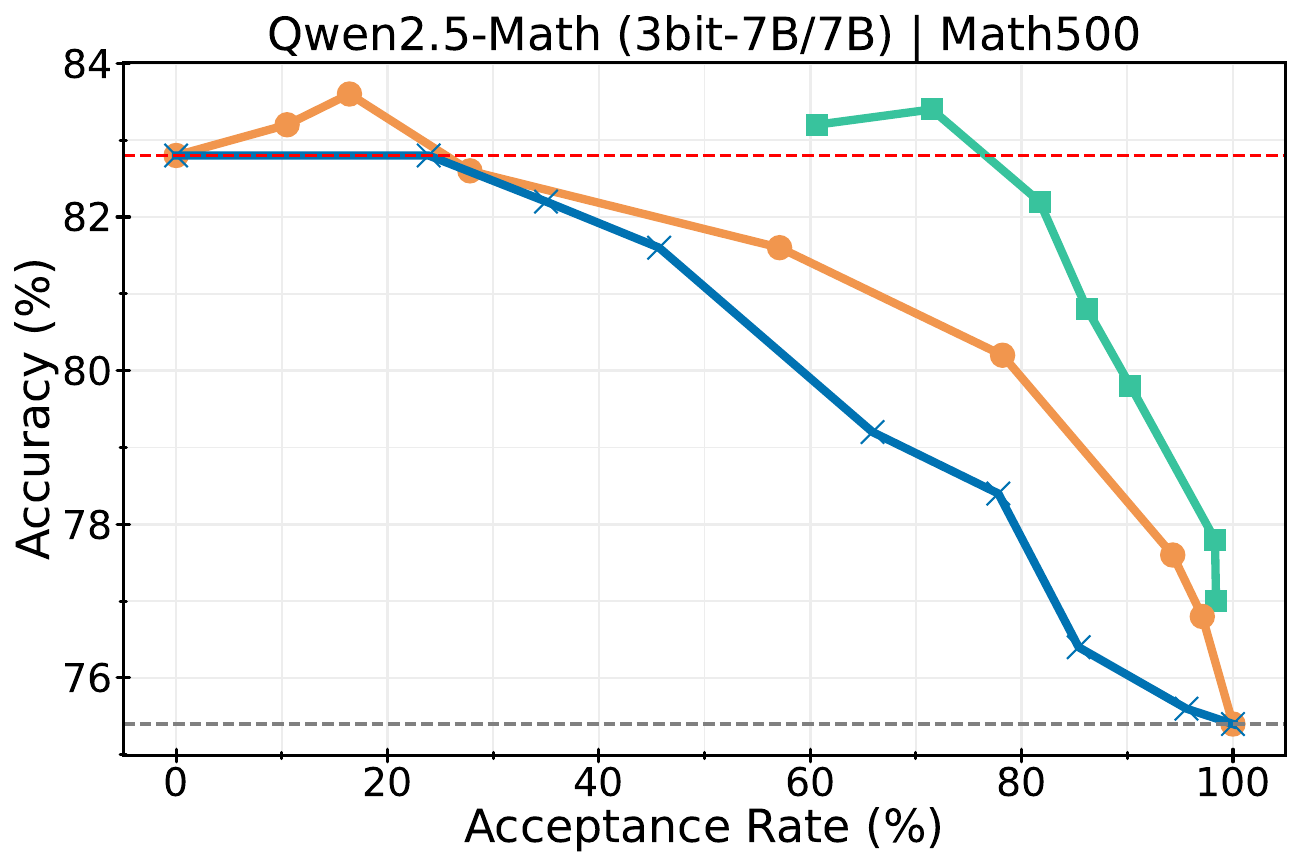}
    \par\vspace{0.25em}
    \includegraphics[width=\linewidth]{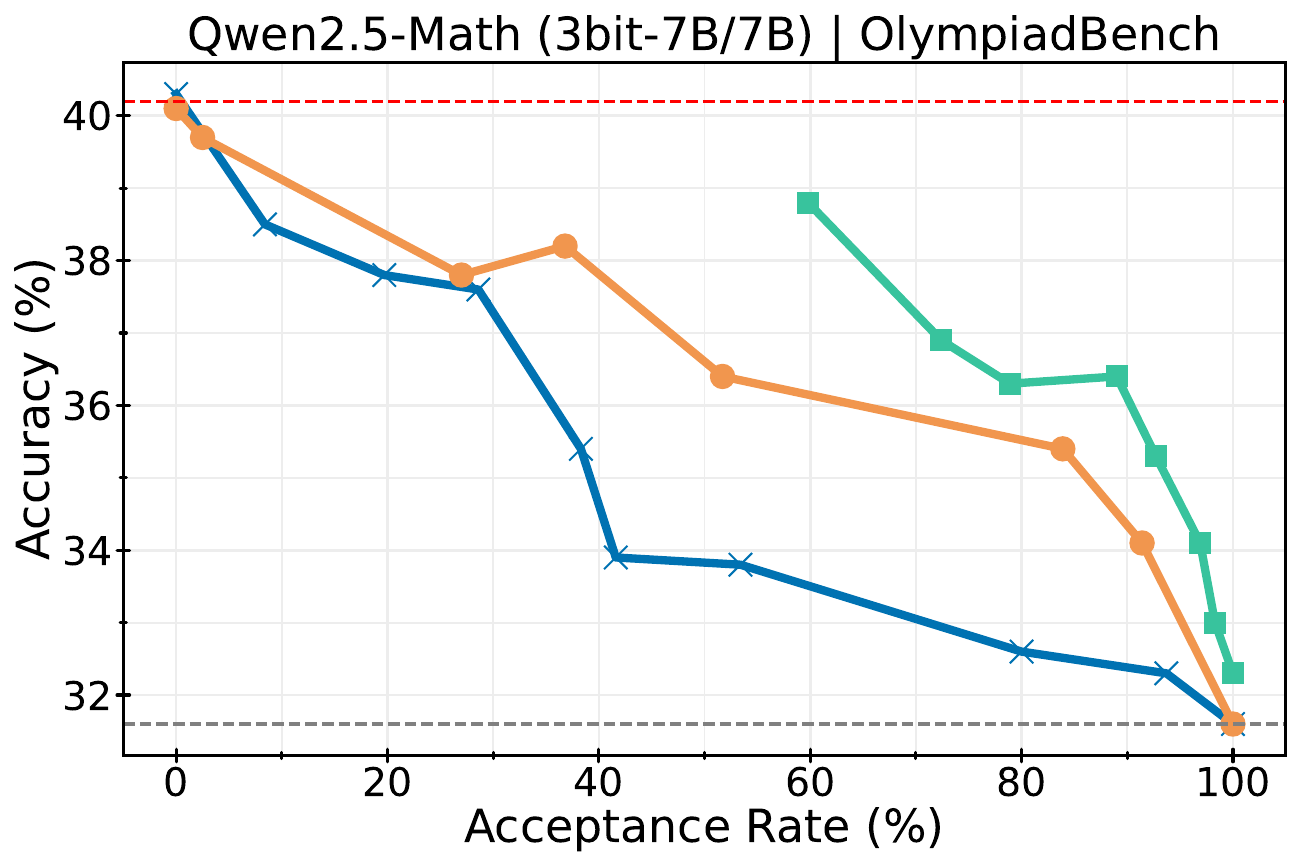}
    \par\vspace{-1pt}\small(c) \
  \end{minipage}

  \caption{\textbf{\OURS improves the compute–quality trade-off.}
  Accuracy vs.\ acceptance rate for \OURSORACLE, \OURSROUTER, and RSD across two benchmarks (MATH500 and OlympiadBench) and three model configurations.
  The top row shows results on MATH500 and the bottom row on OlympiadBench.
  Columns (a), (b), and (c) correspond to LLaMA3 (1B/8B), LLaMA3 (8B/70B), and Qwen2.5-Math (3bit-7B/7B), respectively.
  In all cases, \OURS consistently yields higher accuracy at comparable acceptance rates, demonstrating superior compute–quality efficiency.
  Additional results are provided in Appendix~\ref{app:additional_results}.}
  \label{fig:acc_vs_accept}
\end{figure*}

Figure~\ref{fig:acc_vs_accept} plots accuracy (y-axis) as a function of acceptance rate (x-axis) on MATH500~\cite{hendrycks2021measuring} and OlympiadBench~\cite{he2024olympiadbench} for three routing configurations: LLaMA3 (1B/8B), LLaMA3 (8B/70B), and Qwen2.5-Math (3bit-7B/7B). The top row shows MATH500 and the bottom row OlympiadBench; in each panel we also include the draft-only and target-only accuracies as horizontal reference lines, along with \OURSORACLE as an upper bound. Across all six model–dataset combinations, the \OURSROUTER curve lies strictly above RSD at nearly all acceptance rates, indicating that advantage-aware routing extracts more accuracy per unit of target usage. In these regimes, the \OURSROUTER curve tracks the oracle upper bound closely, while RSD remains noticeably below it. These empirical trends indicate that \OURSROUTER closely tracks the oracle across model scales, datasets, and acceptance regimes.

Looking more closely at individual configurations, we observe that the gains from \OURSROUTER are most pronounced when the draft model is substantially weaker than the target. For example, in the LLaMA3 (1B/8B) setting on both MATH500 and OlympiadBench, RSD remains much closer to the draft-only baseline over a wide range of acceptance rates, whereas \OURSROUTER rapidly climbs toward the target-only line. This suggests that when the potential advantage of escalation is large but unevenly distributed across queries, advantage-aware routing is particularly effective at identifying the subset of examples where the target model is truly needed. In contrast, the LLaMA3 (8B/70B) and Qwen2.5-Math (3bit-7B/7B) settings exhibit smaller draft–target gaps, yet \OURSROUTER still secures a consistent margin over RSD, demonstrating that our method remains beneficial even when the draft is already relatively strong. Near the high-acceptance regime, \OURSROUTER remains closer to the oracle than RSD, implying that it continues to allocate target computation to challenging instances rather than escalating indiscriminately. Finally, the consistent ordering of the methods across both MATH500 and OlympiadBench suggests that the advantages of relative-advantage routing are robust to dataset difficulty and distributional shift, rather than being tied to a particular benchmark.

\subsection{Speedup Analysis}
\label{sec:speedup}
\begin{figure*}[t]
  \centering
  \begin{minipage}[t]{0.49\textwidth}
    \centering
    \includegraphics[width=1.\linewidth]{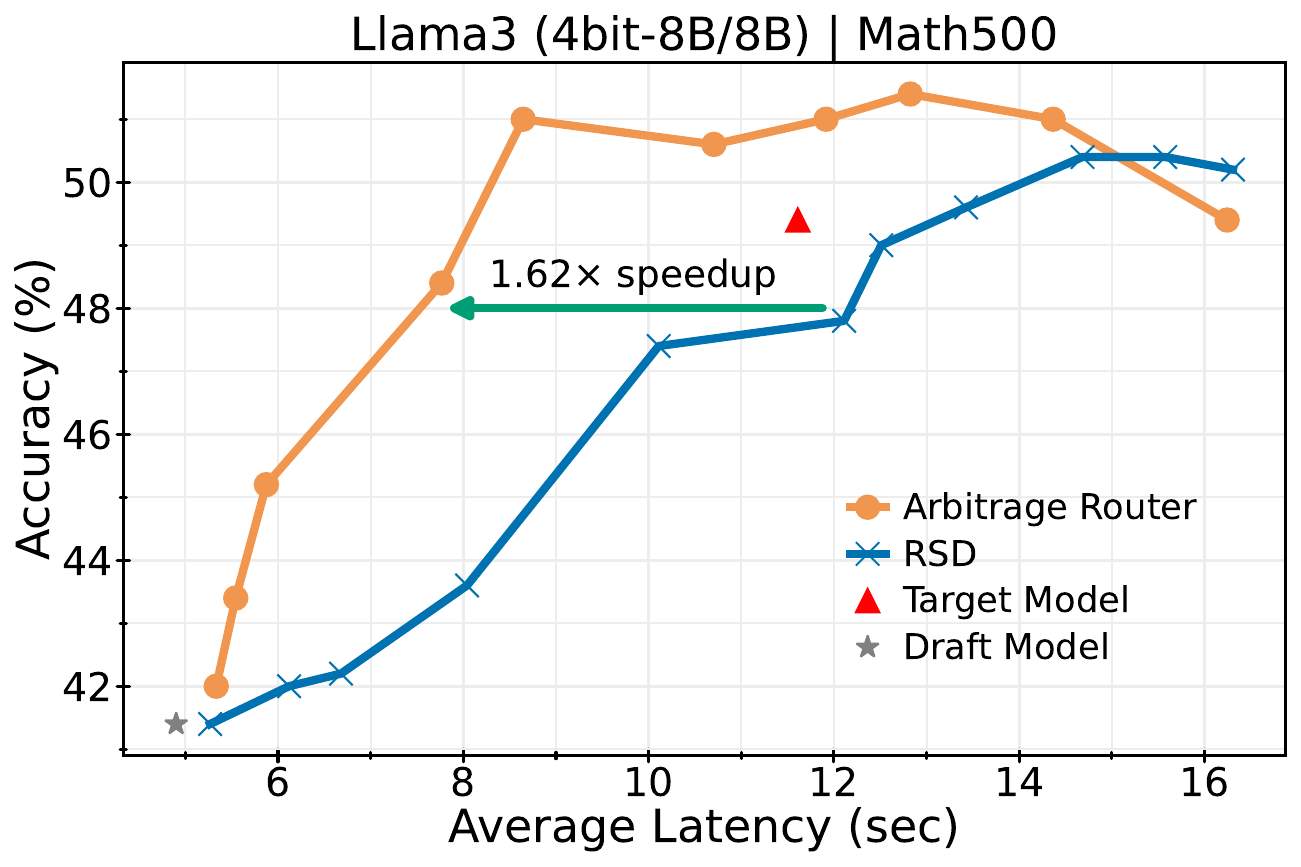}
    \par\small(a)
  \end{minipage}%
  \begin{minipage}[t]{0.49\textwidth}
    \centering
    \includegraphics[width=1.\linewidth]{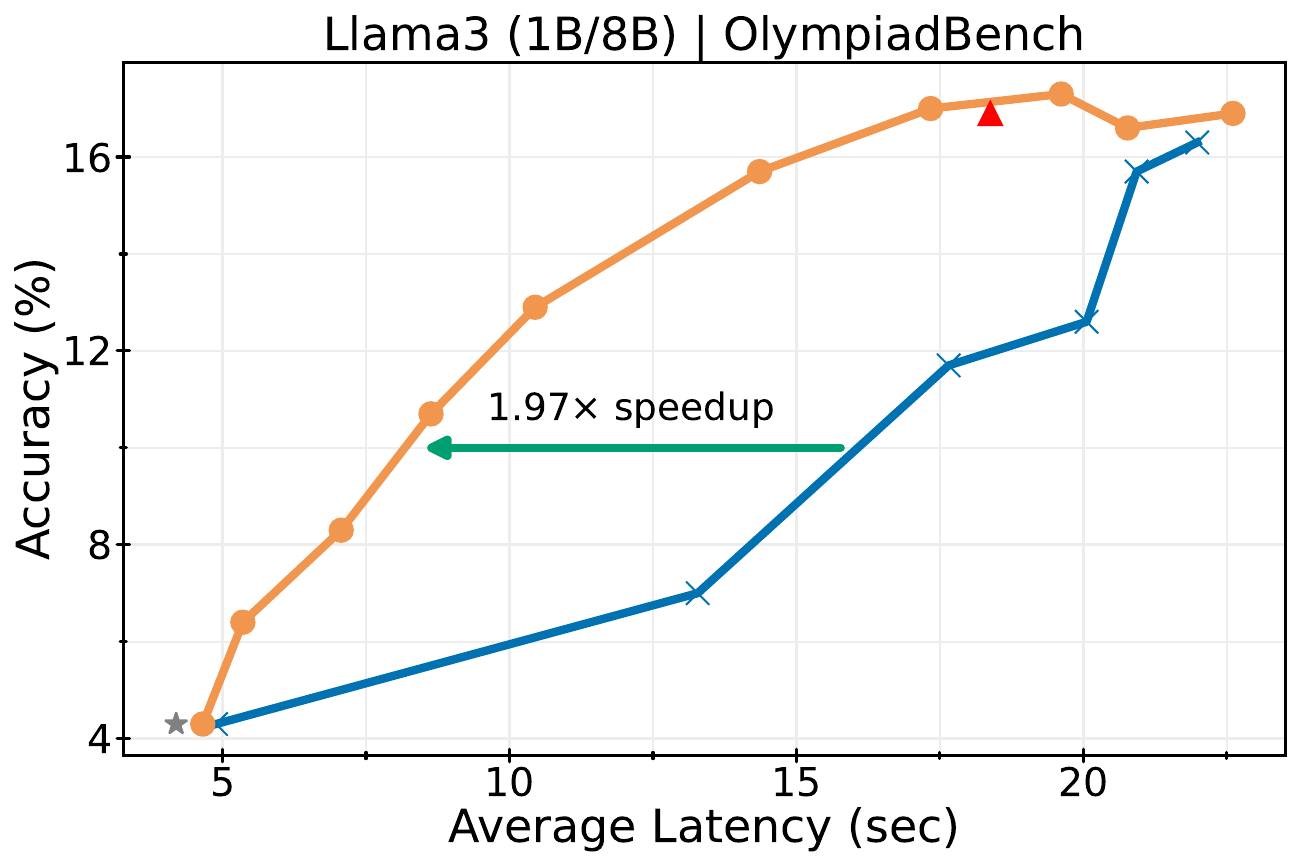}
    \par\small(b)
  \end{minipage}

\caption{\textbf{\OURS improves the compute–quality trade-off.}
Accuracy–time curves for \OURSROUTER and RSD on two LLaMA3 routing configurations.
Subplot (a) reports results for a quantized-draft / full-precision-target setting (Q4-bit-8B/8B/1.5B) on MATH500, and subplot (b) for a small-draft / large-target setting (1B/8B/1.5B) on OlympiadBench. Across both configurations, \OURSROUTER consistently achieves higher accuracy at a given wall-clock time than RSD, yielding a better Pareto frontier.
Each marker corresponds to a different threshold operating point; moving right indicates increased target-model invocations (and thus higher latency).}

\label{fig:acc_vs_time}
\end{figure*}

We quantify the compute savings of \OURS by measuring the average end-to-end wall-clock time per problem as we sweep the routing threshold, yielding the accuracy--latency Pareto curves in Figure~\ref{fig:acc_vs_time}. Each point on a curve corresponds to a different operating threshold, and therefore a different fraction of queries that are escalated from the draft to the target model. Moving to the right along a curve increases the escalation rate and thus the overall latency, since more problems are solved (partially or fully) by the slower target model. Across all model and dataset configurations, \OURS strictly dominates RSD: for any fixed latency budget, \OURS attains higher accuracy, and for any desired accuracy, it achieves a lower latency. On MATH500 under the quantized-draft regime (\textsc{Q4-8B}/8B), \OURS delivers up to \textbf{1.62$\times$} lower latency at comparable accuracy relative to RSD. On OlympiadBench with the small-draft regime (1B/8B), it reaches up to \textbf{1.97$\times$} speedup at matched accuracy. Qualitatively, we observe that the largest gains occur in the mid-latency regime. Here, RSD tends to over-escalate on examples for which the target offers marginal benefit, incurring substantial latency without commensurate accuracy gains. In contrast, \OURS concentrates target compute on instances where the estimated advantage over the draft is highest, and relies on the draft model elsewhere. This selective escalation yields curves that rise more steeply with latency, reflecting a more favorable compute--quality trade-off.

\subsection{Case Studies}
\begin{figure*}[!hbt]
    \centering
    \includegraphics[width=\linewidth]{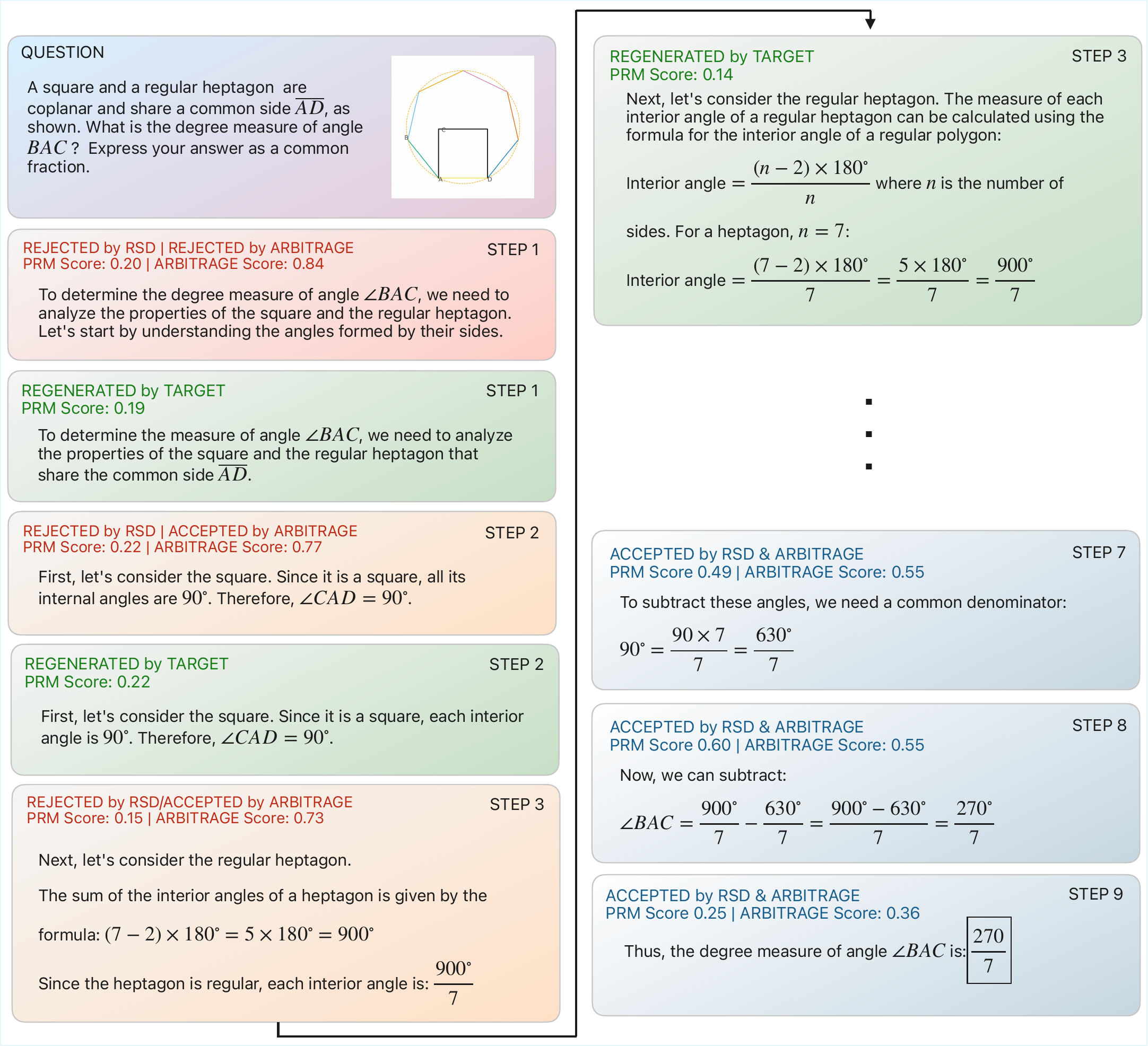}
    \caption{Qualitative example from MATH500. We first obtain a full solution using RSD and then rescore every step with \OURSROUTER under the same acceptance budget. Steps kept by RSD are highlighted in blue, those it rejects are in red or orange, and steps regenerated by the target model are shown in green. Notably, steps 2 and 3 (orange) are accepted by \OURSROUTER but rejected by RSD: their regenerated versions are essentially identical in content and still lead to the correct final answer of $270/7$. The example shows how RSD’s global reward cutoff can discard valid intermediate reasoning, whereas \OURS’s advantage-based, relative routing preserves such steps and eliminates redundant regeneration.}
    \label{fig:case-study}
\end{figure*}

To further illustrate the practical advantages of \OURS over the baseline, we conduct a qualitative case study on the MATH500 benchmark using Qwen2.5-Math 3-bit-7B/7B as the draft and target models. Starting from an RSD-generated trajectory, we re-score each step with \OURSROUTER while enforcing the same 80\% overall acceptance rate for both approaches. As visualized in Figure~\ref{fig:case-study}, RSD rejects several intermediate steps (shown in red and orange) because their individual PRM scores fall below a fixed global cutoff. Many of these steps are in fact correct and encapsulate nontrivial reasoning. For example, determining the interior angles of the regular heptagon and establishing the necessary geometric relationships. When these steps are regenerated by the target model (green), the resulting text is semantically almost identical and still leads to the correct final answer. In contrast, \OURS keeps such steps considering the predicted advantage of the target over the draft: routing decisions are based on the relative benefit of escalation rather than an absolute reward threshold. This advantage-aware view avoids redundant regenerations, preserves coherent reasoning trajectories, and allocates target compute only where a substantially better step is likely. 
\section{Ablations}

We conduct ablation studies to understand how different design choices affect the performance of \OURSROUTER. We focus on two main dimensions: (1) model architecture, and (2) data-related design choices. On the architecture side, we compare multi-class classification against regression-based variants. On the data side, we evaluate the impact of stepwise annotation and class-balanced downsampling. For each setting, we report the Spearman correlation between the router’s predictions and the oracle advantage scores $\Delta$ on the evaluation set, as well as accuracy for both classes. In every ablation, we modify only the factor under study and keep all other components fixed. Some ablations were run mid-development, so their final checkpoints may not be strictly comparable to the main model. Unless otherwise noted, we finetune each variant for 3 epochs.

\subsection{Model Architecture}
\subsubsection*{Label Granularity}
Prior work suggests that supervision with fine-grained labels can improve performance on coarser downstream decisions~\cite{inproceedingsBoosting,articlehierarchical}. Motivated by this, we explore multi-class, ordinal variants of the routing objective. Specifically, we partition the training and evaluation data into 2, 4, or 10 classes based on quantiles of the true advantage score~$\Delta$. Table~\ref{tab:ablation-model-arch} summarizes the results. The 2-class setting achieves the highest overall Spearman correlation and maintains balanced accuracy across the two labels (accept: label 0 vs.\ escalate: label 1), indicating that it provides a good trade-off between acceptance rate and end-task quality. When we increase the number of classes to 4 or 10, overall correlation drops or the accuracy becomes skewed between label~0 and label~1, leading to worse routing behavior. Empirically, the advantage scores are densely concentrated around zero with a long positive tail. In this regime, aggressive discretization into many bins introduces ambiguous decision boundaries and effectively blurs neighboring classes. Given these observations, we adopt the 2-class formulation as our default.

\subsubsection*{Classification vs. Regression}  
Because the advantage score $\Delta$ is continuous, a natural alternative to classification is to train a regression model. For ordered labels, ordinal regression is often used to exploit the underlying ranking structure~\cite{wang2025surveyordinalregressionapplications}. We therefore compare a standard 4-way classification model with an ordinal-regression model defined over the same 4 ordinal levels. As shown in Table~\ref{tab:ablation-model-arch}, the ordinal regression variant underperforms the classification model across all metrics, despite using richer ranking information during training. In contrast, the classification architecture yields stronger empirical performance and is easier to calibrate and deploy. Based on this combination of accuracy and practical simplicity, we use the standard classification formulation in the final \OURSROUTER{}.
\begin{table}[h]
\caption{Ablation results for different model architectures, varying label granularity and including an ordinal regression variant. The standard 2-class formulation (accept vs.\ escalate) attains the best trade-off between Spearman correlation and per-label accuracy, indicating that binary routing is the most robust choice.}
\centering
\begin{tabular}{lccc}
\toprule
Variant  & Spearman ($\rho$) & Acc 0 & Acc 1 \\
\midrule
\rowcolor{blue!8}
2-Classification               & 0.1508 & 70.04 & 72.96 \\
4-Classification               & 0.1417 & 64.24 & 80.53 \\
4-Class Ordinal     & 0.1355 & 62.00 & 79.77 \\
10-Classification              & 0.1337 & 71.29 & 71.77 \\
\bottomrule
\end{tabular}
\label{tab:ablation-model-arch}
\end{table}

\subsection{Data-related Design Choices}
\subsubsection*{Step Annotation}
Another design choice is whether to incorporate explicit signals about the reasoning trajectory at each step. A baseline approach treats each step independently, predicting the routing decision from the local step context alone. Our hypothesis is that previous routing decisions provide valuable context: if the target model has been invoked repeatedly in earlier steps, the current step is more likely to belong to a difficult portion of the problem and thus may benefit from escalation.
To test this hypothesis, we introduce a simple annotation scheme that encodes the history of model choices. For each step $i$, we prepend a tag such as \texttt{Model 0} or \texttt{Model 1} to the step input, indicating which model was used at that step. This produces contextualized inputs whose annotations summarize both the overall difficulty of the question and the local difficulty pattern along the chain of thought. Table~\ref{tab:ablation-annotation} shows that including these annotations improves label-1 accuracy (correctly escalating when the target model is beneficial) and increases Spearman correlation, which correlates with end-to-end task accuracy. Because these gains are consistent and the mechanism is simple to implement, we enable step annotations by default in \OURSROUTER{}.
\begin{table}[h]
\caption{Effect of incorporating historical routing context. The annotated setting, which conditions each step on prior routing decisions, improves both Spearman correlation and label-1 accuracy, indicating that leveraging routing history enhances end-to-end routing performance.}

\centering
\begin{tabular}{lccc}
\toprule
Variant  & Spearman ($\rho$) & Acc 0 & Acc 1 \\
\midrule
\rowcolor{blue!8}
Annotated               & 0.1508 & 70.04 & 72.96 \\
Not Annotated           & 0.1305 & 73.06 & 69.58 \\
\bottomrule
\end{tabular}
\label{tab:ablation-annotation}
\end{table}

\subsubsection*{Data Downsampling}  
As discussed in Section~\ref{sec:preprocess}, the training data for \OURSROUTER{} is highly imbalanced: negative examples (steps where the draft is accepted, $y=0$) occur roughly twice as often as positive examples (steps where escalation is preferred, $y=1$). This reflects the fact that most steps are sufficiently handled by the draft model, with only a minority requiring escalation.
We evaluate two simple preprocessing strategies to address this imbalance:
(1) \emph{No downsampling}, where we train directly on the raw, imbalanced dataset; and (2) \emph{Balanced downsampling}, where we randomly subsample accepted steps ($y = 0$) so that the counts of $y = 0$ and $y = 1$ are approximately equal.
Results are reported in Table~\ref{tab:ablation-downsampling}. Without downsampling, the router becomes overconfident in predicting $\hat{y} = 0$ and assigns very low probability to $\hat{y} = 1$, leading to under-escalation and suboptimal performance. With balanced downsampling, even though we train on fewer total examples, the model exhibits better calibration and produces more balanced predictions over the two classes. In our final training pipeline, we therefore apply class-balanced downsampling to mitigate label skew and improve routing quality.
\begin{table}[h]
\caption{Effect of class-balanced downsampling. Balancing the training data reduces bias toward the majority (accept) class, improving Spearman correlation and label-1 accuracy and leading to more stable routing behavior.}

\centering
\begin{tabular}{lccc}
\toprule
Variant  & Spearman ($\rho$) & Acc 0 & Acc 1 \\
\midrule
\rowcolor{blue!8}
Balanced DS          & 0.1673 & 63.28 & 64.80 \\
No DS    & 0.1587 & 77.34 & 49.62 \\
\bottomrule
\end{tabular}
\label{tab:ablation-downsampling}
\end{table}

\section{Conclusions}

We study step-level speculative generation as a paradigm for accelerating LLM inference on reasoning-intensive tasks, and we identify a key source of inefficiency in existing approaches. In current methods, when a draft step is rejected and recomputed by the target model, the regenerated step often provides limited or no improvement, while still incurring the full latency cost of the target model. To address this, we introduce \OURS, a speculative generation framework that explicitly reasons about the relative quality of draft and target steps. \OURS has two components:
(1) the \OURSORACLE, an idealized routing policy that selects the better of the two steps using ground-truth advantage signals; and
(2) the \OURSROUTER, a lightweight, trainable predictor that approximates this oracle using only draft-generated context, enabling practical deployment without querying the target model during routing. Across multiple mathematical reasoning benchmarks, \OURS consistently improves over prior baselines in both inference speed and answer quality. For example, on MATH500 with a quantized-draft regime (Q4-8B/8B), \OURS achieves up to $1.62\times$ lower latency at comparable accuracy, and on OlympiadBench with a small-draft regime (1B/8B), it reaches up to $1.97\times$ speedup at matched accuracy. More broadly, \OURS reduces end-to-end latency by up to roughly $2\times$ over step-level speculative decoding baselines at fixed accuracy targets. By replacing absolute, draft-only acceptance rules with an expected-advantage estimate, \OURS achieves a stronger efficiency–accuracy trade-off and establishes a new baseline for step-level speculative decoding.

\section*{Acknowledgements}
We acknowledge gracious support from the Furiosa AI, Intel, Apple, NVIDIA, Macronix, and Mozilla team.
Furthermore, we appreciate support from
Google Cloud, the Google TRC team and Prof. David Patterson.
Prof. Keutzer's lab is sponsored by the Intel corporation, UC Berkeley oneAPI Center of Excellence, Intel VLAB team, as well as funding through BDD and BAIR.
We also acknowledge support by the Director, Office of Science, Office of Advanced Scientific Computing Research, of the U.S. Department of Energy under Contract No. DE-AC02-05CH11231.
MWM would also like to acknowledge DARPA, DOE, NSF, and ONR.
DOE SciGPT grant. Our conclusions do not necessarily reflect the position or the policy of our sponsors, and no official endorsement should be~inferred.

\bibliographystyle{plain}
\bibliography{references}

@misc{liao2025rewardguidedspeculativedecodingefficient,
      title={Reward-Guided Speculative Decoding for Efficient LLM Reasoning}, 
      author={Baohao Liao and Yuhui Xu and Hanze Dong and Junnan Li and Christof Monz and Silvio Savarese and Doyen Sahoo and Caiming Xiong},
      year={2025},
      eprint={2501.19324},
      archivePrefix={arXiv},
      primaryClass={cs.CL},
      url={https://arxiv.org/abs/2501.19324}, 
}

@misc{wei2023chainofthoughtpromptingelicitsreasoning,
      title={Chain-of-Thought Prompting Elicits Reasoning in Large Language Models}, 
      author={Jason Wei and Xuezhi Wang and Dale Schuurmans and Maarten Bosma and Brian Ichter and Fei Xia and Ed Chi and Quoc Le and Denny Zhou},
      year={2023},
      eprint={2201.11903},
      archivePrefix={arXiv},
      primaryClass={cs.CL},
      url={https://arxiv.org/abs/2201.11903}, 
}

@misc{snell2024scalingllmtesttimecompute,
      title={Scaling LLM Test-Time Compute Optimally can be More Effective than Scaling Model Parameters}, 
      author={Charlie Snell and Jaehoon Lee and Kelvin Xu and Aviral Kumar},
      year={2024},
      eprint={2408.03314},
      archivePrefix={arXiv},
      primaryClass={cs.LG},
      url={https://arxiv.org/abs/2408.03314}, 
}

@misc{zhang2024accessinggpt4levelmathematical,
      title={Accessing GPT-4 level Mathematical Olympiad Solutions via Monte Carlo Tree Self-refine with LLaMa-3 8B}, 
      author={Di Zhang and Xiaoshui Huang and Dongzhan Zhou and Yuqiang Li and Wanli Ouyang},
      year={2024},
      eprint={2406.07394},
      archivePrefix={arXiv},
      primaryClass={cs.AI},
      url={https://arxiv.org/abs/2406.07394}, 
}

@misc{muennighoff2025s1simpletesttimescaling,
      title={s1: Simple test-time scaling}, 
      author={Niklas Muennighoff and Zitong Yang and Weijia Shi and Xiang Lisa Li and Li Fei-Fei and Hannaneh Hajishirzi and Luke Zettlemoyer and Percy Liang and Emmanuel Candès and Tatsunori Hashimoto},
      year={2025},
      eprint={2501.19393},
      archivePrefix={arXiv},
      primaryClass={cs.CL},
      url={https://arxiv.org/abs/2501.19393}, 
}

@misc{lightman2023letsverifystepstep,
      title={Let's Verify Step by Step}, 
      author={Hunter Lightman and Vineet Kosaraju and Yura Burda and Harri Edwards and Bowen Baker and Teddy Lee and Jan Leike and John Schulman and Ilya Sutskever and Karl Cobbe},
      year={2023},
      eprint={2305.20050},
      archivePrefix={arXiv},
      primaryClass={cs.LG},
      url={https://arxiv.org/abs/2305.20050}, 
}

@misc{xu2025genarmrewardguidedgeneration,
      title={GenARM: Reward Guided Generation with Autoregressive Reward Model for Test-time Alignment}, 
      author={Yuancheng Xu and Udari Madhushani Sehwag and Alec Koppel and Sicheng Zhu and Bang An and Furong Huang and Sumitra Ganesh},
      year={2025},
      eprint={2410.08193},
      archivePrefix={arXiv},
      primaryClass={cs.CL},
      url={https://arxiv.org/abs/2410.08193}, 
}

@misc{zhao2025genprmscalingtesttimecompute,
      title={GenPRM: Scaling Test-Time Compute of Process Reward Models via Generative Reasoning}, 
      author={Jian Zhao and Runze Liu and Kaiyan Zhang and Zhimu Zhou and Junqi Gao and Dong Li and Jiafei Lyu and Zhouyi Qian and Biqing Qi and Xiu Li and Bowen Zhou},
      year={2025},
      eprint={2504.00891},
      archivePrefix={arXiv},
      primaryClass={cs.CL},
      url={https://arxiv.org/abs/2504.00891}, 
}

@misc{pang2025boltbootstraplongchainofthought,
      title={BOLT: Bootstrap Long Chain-of-Thought in Language Models without Distillation}, 
      author={Bo Pang and Hanze Dong and Jiacheng Xu and Silvio Savarese and Yingbo Zhou and Caiming Xiong},
      year={2025},
      eprint={2502.03860},
      archivePrefix={arXiv},
      primaryClass={cs.CL},
      url={https://arxiv.org/abs/2502.03860}, 
}

@misc{yeo2025demystifyinglongchainofthoughtreasoning,
      title={Demystifying Long Chain-of-Thought Reasoning in LLMs}, 
      author={Edward Yeo and Yuxuan Tong and Morry Niu and Graham Neubig and Xiang Yue},
      year={2025},
      eprint={2502.03373},
      archivePrefix={arXiv},
      primaryClass={cs.CL},
      url={https://arxiv.org/abs/2502.03373}, 
}

@misc{deepseekai2025deepseekr1incentivizingreasoningcapability,
      title={DeepSeek-R1: Incentivizing Reasoning Capability in LLMs via Reinforcement Learning}, 
      author={DeepSeek-AI Team},
      year={2025},
      eprint={2501.12948},
      archivePrefix={arXiv},
      primaryClass={cs.CL},
      url={https://arxiv.org/abs/2501.12948}, 
}

@misc{shao2024deepseekmathpushinglimitsmathematical,
      title={DeepSeekMath: Pushing the Limits of Mathematical Reasoning in Open Language Models}, 
      author={Zhihong Shao and Peiyi Wang and Qihao Zhu and Runxin Xu and Junxiao Song and Xiao Bi and Haowei Zhang and Mingchuan Zhang and Y. K. Li and Y. Wu and Daya Guo},
      year={2024},
      eprint={2402.03300},
      archivePrefix={arXiv},
      primaryClass={cs.CL},
      url={https://arxiv.org/abs/2402.03300}, 
}

@misc{yang2024qwen25mathtechnicalreportmathematical,
      title={Qwen2.5-Math Technical Report: Toward Mathematical Expert Model via Self-Improvement}, 
      author={An Yang and Beichen Zhang and Binyuan Hui and Bofei Gao and Bowen Yu and Chengpeng Li and Dayiheng Liu and Jianhong Tu and Jingren Zhou and Junyang Lin and Keming Lu and Mingfeng Xue and Runji Lin and Tianyu Liu and Xingzhang Ren and Zhenru Zhang},
      year={2024},
      eprint={2409.12122},
      archivePrefix={arXiv},
      primaryClass={cs.CL},
      url={https://arxiv.org/abs/2409.12122}, 
}

@misc{zhou2024distillspecimprovingspeculativedecoding,
      title={DistillSpec: Improving Speculative Decoding via Knowledge Distillation}, 
      author={Yongchao Zhou and Kaifeng Lyu and Ankit Singh Rawat and Aditya Krishna Menon and Afshin Rostamizadeh and Sanjiv Kumar and Jean-François Kagy and Rishabh Agarwal},
      year={2024},
      eprint={2310.08461},
      archivePrefix={arXiv},
      primaryClass={cs.CL},
      url={https://arxiv.org/abs/2310.08461}, 
}

@misc{tiwari2025quantspecselfspeculativedecodinghierarchical,
      title={QuantSpec: Self-Speculative Decoding with Hierarchical Quantized KV Cache}, 
      author={Rishabh Tiwari and Haocheng Xi and Aditya Tomar and Coleman Hooper and Sehoon Kim and Maxwell Horton and Mahyar Najibi and Michael W. Mahoney and Kurt Keutzer and Amir Gholami},
      year={2025},
      eprint={2502.10424},
      archivePrefix={arXiv},
      primaryClass={cs.LG},
      url={https://arxiv.org/abs/2502.10424}, 
}

@misc{cai2024medusasimplellminference,
      title={Medusa: Simple LLM Inference Acceleration Framework with Multiple Decoding Heads}, 
      author={Tianle Cai and Yuhong Li and Zhengyang Geng and Hongwu Peng and Jason D. Lee and Deming Chen and Tri Dao},
      year={2024},
      eprint={2401.10774},
      archivePrefix={arXiv},
      primaryClass={cs.LG},
      url={https://arxiv.org/abs/2401.10774}, 
}

@misc{yang2025speculativethinkingenhancingsmallmodel,
      title={Speculative Thinking: Enhancing Small-Model Reasoning with Large Model Guidance at Inference Time}, 
      author={Wang Yang and Xiang Yue and Vipin Chaudhary and Xiaotian Han},
      year={2025},
      eprint={2504.12329},
      archivePrefix={arXiv},
      primaryClass={cs.CL},
      url={https://arxiv.org/abs/2504.12329}, 
}

@misc{loshchilov2019decoupledweightdecayregularization,
      title={Decoupled Weight Decay Regularization}, 
      author={Ilya Loshchilov and Frank Hutter},
      year={2019},
      eprint={1711.05101},
      archivePrefix={arXiv},
      primaryClass={cs.LG},
      url={https://arxiv.org/abs/1711.05101}, 
}

@inbook{spearmanpaper,
author = {Zar, Jerrold},
year = {2005},
month = {07},
pages = {},
title = {Spearman Rank Correlation},
volume = {5},
isbn = {9780470011812},
journal = {Encycl Biostat},
doi = {10.1002/0470011815.b2a15150}
}

@misc{guha2025openthoughtsdatarecipesreasoning,
      title={OpenThoughts: Data Recipes for Reasoning Models}, 
      author={Etash Guha et. al.},
      year={2025},
      eprint={2506.04178},
      archivePrefix={arXiv},
      primaryClass={cs.LG},
      url={https://arxiv.org/abs/2506.04178}, 
}

@misc{zhou2025megamathpushinglimitsopen,
      title={MegaMath: Pushing the Limits of Open Math Corpora}, 
      author={Fan Zhou and Zengzhi Wang and Nikhil Ranjan and Zhoujun Cheng and Liping Tang and Guowei He and Zhengzhong Liu and Eric P. Xing},
      year={2025},
      eprint={2504.02807},
      archivePrefix={arXiv},
      primaryClass={cs.CL},
      url={https://arxiv.org/abs/2504.02807}, 
}

@misc{zheng2025groupsequencepolicyoptimization,
      title={Group Sequence Policy Optimization}, 
      author={Chujie Zheng and Shixuan Liu and Mingze Li and Xiong-Hui Chen and Bowen Yu and Chang Gao and Kai Dang and Yuqiong Liu and Rui Men and An Yang and Jingren Zhou and Junyang Lin},
      year={2025},
      eprint={2507.18071},
      archivePrefix={arXiv},
      primaryClass={cs.LG},
      url={https://arxiv.org/abs/2507.18071}, 
}

@misc{frantar2023gptqaccurateposttrainingquantization,
      title={GPTQ: Accurate Post-Training Quantization for Generative Pre-trained Transformers}, 
      author={Elias Frantar and Saleh Ashkboos and Torsten Hoefler and Dan Alistarh},
      year={2023},
      eprint={2210.17323},
      archivePrefix={arXiv},
      primaryClass={cs.LG},
      url={https://arxiv.org/abs/2210.17323}, 
}

@misc{leviathan2023fastinferencetransformersspeculative,
      title={Fast Inference from Transformers via Speculative Decoding}, 
      author={Yaniv Leviathan and Matan Kalman and Yossi Matias},
      year={2023},
      eprint={2211.17192},
      archivePrefix={arXiv},
      primaryClass={cs.LG},
      url={https://arxiv.org/abs/2211.17192}, 
}

@misc{sun2025challengingboundariesreasoningolympiadlevel,
      title={Challenging the Boundaries of Reasoning: An Olympiad-Level Math Benchmark for Large Language Models}, 
      author={Haoxiang Sun and Yingqian Min and Zhipeng Chen and Wayne Xin Zhao and Lei Fang and Zheng Liu and Zhongyuan Wang and Ji-Rong Wen},
      year={2025},
      eprint={2503.21380},
      archivePrefix={arXiv},
      primaryClass={cs.CL},
      url={https://arxiv.org/abs/2503.21380}, 
}

@misc{zhu2025oibenchbenchmarkingstrongreasoning,
      title={OIBench: Benchmarking Strong Reasoning Models with Olympiad in Informatics}, 
      author={Yaoming Zhu and Junxin Wang and Yiyang Li and Lin Qiu and ZongYu Wang and Jun Xu and Xuezhi Cao and Yuhuai Wei and Mingshi Wang and Xunliang Cai and Rong Ma},
      year={2025},
      eprint={2506.10481},
      archivePrefix={arXiv},
      primaryClass={cs.AI},
      url={https://arxiv.org/abs/2506.10481}, 
}

@misc{he_2024_16998085,
  author       = {He, Jujie and
                  Wei, Tianwen and
                  Yan, Rui and
                  Liu, Jiacai and
                  Wang, Chaojie and
                  Gan, Yimeng and
                  Tu, Shiwen and
                  Liu, Chris Yuhao and
                  Zeng, Liang and
                  Wang, Xiaokun and
                  Wang, Boyang and
                  Li, Yongcong and
                  Zhang, Fuxiang and
                  Xu, Jiacheng and
                  An, Bo and
                  Liu, Yang and
                  Zhou, Yahui},
  title        = {Skywork-o1 Open Series},
  month        = nov,
  year         = 2024,
  publisher    = {Zenodo},
  version      = {1.0.0},
  doi          = {10.5281/zenodo.16998085},
  url          = {https://doi.org/10.5281/zenodo.16998085},
}

@inproceedings{wang2024math,
  title={Math-shepherd: Verify and reinforce llms step-by-step without human annotations},
  author={Wang, Peiyi and Li, Lei and Shao, Zhihong and Xu, Runxin and Dai, Damai and Li, Yifei and Chen, Deli and Wu, Yu and Sui, Zhifang},
  booktitle={Proceedings of the 62nd Annual Meeting of the Association for Computational Linguistics (Volume 1: Long Papers)},
  pages={9426--9439},
  year={2024}
}

@misc{gao2024omnimathuniversalolympiadlevel,
      title={Omni-MATH: A Universal Olympiad Level Mathematic Benchmark For Large Language Models}, 
      author={Bofei Gao and Feifan Song and Zhe Yang and Zefan Cai and Yibo Miao and Qingxiu Dong and Lei Li and Chenghao Ma and Liang Chen and Runxin Xu and Zhengyang Tang and Benyou Wang and Daoguang Zan and Shanghaoran Quan and Ge Zhang and Lei Sha and Yichang Zhang and Xuancheng Ren and Tianyu Liu and Baobao Chang},
      year={2024},
      eprint={2410.07985},
      archivePrefix={arXiv},
      primaryClass={cs.CL},
      url={https://arxiv.org/abs/2410.07985}, 
}

@misc{openr1-math-220k,
    title = {OpenR1-Math-220k},
    author = {{Open R1 Project}},
    year = {2025},
    publisher = {Hugging Face},
    url = {https://huggingface.co/datasets/open-r1/OpenR1-Math-220k}
}

@misc{qwq32b,
    title = {QwQ-32B: Embracing the Power of Reinforcement Learning},
    url = {https://qwenlm.github.io/blog/qwq-32b/},
    author = {Qwen Team},
    month = {March},
    year = {2025}
}

@misc{yang2025qwen3technicalreport,
      title={Qwen3 Technical Report}, 
      author={An Yang and Anfeng Li and Baosong Yang and Beichen Zhang and Binyuan Hui and Bo Zheng and Bowen Yu and Chang Gao and Chengen Huang and Chenxu Lv and Chujie Zheng and Dayiheng Liu and Fan Zhou and Fei Huang and Feng Hu and Hao Ge and Haoran Wei and Huan Lin and Jialong Tang and Jian Yang and Jianhong Tu and Jianwei Zhang and Jianxin Yang and Jiaxi Yang and Jing Zhou and Jingren Zhou and Junyang Lin and Kai Dang and Keqin Bao and Kexin Yang and Le Yu and Lianghao Deng and Mei Li and Mingfeng Xue and Mingze Li and Pei Zhang and Peng Wang and Qin Zhu and Rui Men and Ruize Gao and Shixuan Liu and Shuang Luo and Tianhao Li and Tianyi Tang and Wenbiao Yin and Xingzhang Ren and Xinyu Wang and Xinyu Zhang and Xuancheng Ren and Yang Fan and Yang Su and Yichang Zhang and Yinger Zhang and Yu Wan and Yuqiong Liu and Zekun Wang and Zeyu Cui and Zhenru Zhang and Zhipeng Zhou and Zihan Qiu},
      year={2025},
      eprint={2505.09388},
      archivePrefix={arXiv},
      primaryClass={cs.CL},
      url={https://arxiv.org/abs/2505.09388}, 
}

@misc{yu2025dapoopensourcellmreinforcement,
      title={DAPO: An Open-Source LLM Reinforcement Learning System at Scale}, 
      author={Qiying Yu and Zheng Zhang and Ruofei Zhu and Yufeng Yuan and Xiaochen Zuo and Yu Yue and Weinan Dai and Tiantian Fan and Gaohong Liu and Lingjun Liu and Xin Liu and Haibin Lin and Zhiqi Lin and Bole Ma and Guangming Sheng and Yuxuan Tong and Chi Zhang and Mofan Zhang and Wang Zhang and Hang Zhu and Jinhua Zhu and Jiaze Chen and Jiangjie Chen and Chengyi Wang and Hongli Yu and Yuxuan Song and Xiangpeng Wei and Hao Zhou and Jingjing Liu and Wei-Ying Ma and Ya-Qin Zhang and Lin Yan and Mu Qiao and Yonghui Wu and Mingxuan Wang},
      year={2025},
      eprint={2503.14476},
      archivePrefix={arXiv},
      primaryClass={cs.LG},
      url={https://arxiv.org/abs/2503.14476}, 
}

@misc{he2024olympiadbenchchallengingbenchmarkpromoting,
      title={OlympiadBench: A Challenging Benchmark for Promoting AGI with Olympiad-Level Bilingual Multimodal Scientific Problems}, 
      author={Chaoqun He and Renjie Luo and Yuzhuo Bai and Shengding Hu and Zhen Leng Thai and Junhao Shen and Jinyi Hu and Xu Han and Yujie Huang and Yuxiang Zhang and Jie Liu and Lei Qi and Zhiyuan Liu and Maosong Sun},
      year={2024},
      eprint={2402.14008},
      archivePrefix={arXiv},
      primaryClass={cs.CL},
      url={https://arxiv.org/abs/2402.14008}, 
}

@misc{jain2024livecodebenchholisticcontaminationfree,
      title={LiveCodeBench: Holistic and Contamination Free Evaluation of Large Language Models for Code}, 
      author={Naman Jain and King Han and Alex Gu and Wen-Ding Li and Fanjia Yan and Tianjun Zhang and Sida Wang and Armando Solar-Lezama and Koushik Sen and Ion Stoica},
      year={2024},
      eprint={2403.07974},
      archivePrefix={arXiv},
      primaryClass={cs.SE},
      url={https://arxiv.org/abs/2403.07974}, 
}

@misc{chen2023acceleratinglargelanguagemodel,
      title={Accelerating Large Language Model Decoding with Speculative Sampling}, 
      author={Charlie Chen and Sebastian Borgeaud and Geoffrey Irving and Jean-Baptiste Lespiau and Laurent Sifre and John Jumper},
      year={2023},
      eprint={2302.01318},
      archivePrefix={arXiv},
      primaryClass={cs.CL},
      url={https://arxiv.org/abs/2302.01318}, 
}

@misc{hu2025adaspecselectiveknowledgedistillation,
      title={AdaSPEC: Selective Knowledge Distillation for Efficient Speculative Decoders}, 
      author={Yuezhou Hu and Jiaxin Guo and Xinyu Feng and Tuo Zhao},
      year={2025},
      eprint={2510.19779},
      archivePrefix={arXiv},
      primaryClass={cs.CL},
      url={https://arxiv.org/abs/2510.19779}, 
}

@misc{li2025eagle3scalinginferenceacceleration,
      title={EAGLE-3: Scaling up Inference Acceleration of Large Language Models via Training-Time Test}, 
      author={Yuhui Li and Fangyun Wei and Chao Zhang and Hongyang Zhang},
      year={2025},
      eprint={2503.01840},
      archivePrefix={arXiv},
      primaryClass={cs.CL},
      url={https://arxiv.org/abs/2503.01840}, 
}

@misc{kim2023speculativedecodingbiglittle,
      title={Speculative Decoding with Big Little Decoder}, 
      author={Sehoon Kim and Karttikeya Mangalam and Suhong Moon and Jitendra Malik and Michael W. Mahoney and Amir Gholami and Kurt Keutzer},
      year={2023},
      eprint={2302.07863},
      archivePrefix={arXiv},
      primaryClass={cs.CL},
      url={https://arxiv.org/abs/2302.07863}, 
}

@misc{liu2024onlinespeculativedecoding,
      title={Online Speculative Decoding}, 
      author={Xiaoxuan Liu and Lanxiang Hu and Peter Bailis and Alvin Cheung and Zhijie Deng and Ion Stoica and Hao Zhang},
      year={2024},
      eprint={2310.07177},
      archivePrefix={arXiv},
      primaryClass={cs.AI},
      url={https://arxiv.org/abs/2310.07177}, 
}

@misc{ankner2024hydrasequentiallydependentdraftheads,
      title={Hydra: Sequentially-Dependent Draft Heads for Medusa Decoding}, 
      author={Zachary Ankner and Rishab Parthasarathy and Aniruddha Nrusimha and Christopher Rinard and Jonathan Ragan-Kelley and William Brandon},
      year={2024},
      eprint={2402.05109},
      archivePrefix={arXiv},
      primaryClass={cs.LG},
      url={https://arxiv.org/abs/2402.05109}, 
}

@inproceedings{
zhu2024starlingb,
title={Starling-7B: Improving Helpfulness and Harmlessness with {RLAIF}},
author={Banghua Zhu and Evan Frick and Tianhao Wu and Hanlin Zhu and Karthik Ganesan and Wei-Lin Chiang and Jian Zhang and Jiantao Jiao},
booktitle={First Conference on Language Modeling},
year={2024},
url={https://openreview.net/forum?id=GqDntYTTbk}
}

@misc{liu2024skyworkrewardbagtricksreward,
      title={Skywork-Reward: Bag of Tricks for Reward Modeling in LLMs}, 
      author={Chris Yuhao Liu and Liang Zeng and Jiacai Liu and Rui Yan and Jujie He and Chaojie Wang and Shuicheng Yan and Yang Liu and Yahui Zhou},
      year={2024},
      eprint={2410.18451},
      archivePrefix={arXiv},
      primaryClass={cs.AI},
      url={https://arxiv.org/abs/2410.18451}, 
}

@misc{zhu2025retrievalaugmentedprocessrewardmodel,
      title={Retrieval-Augmented Process Reward Model for Generalizable Mathematical Reasoning}, 
      author={Jiachen Zhu and Congmin Zheng and Jianghao Lin and Kounianhua Du and Ying Wen and Yong Yu and Jun Wang and Weinan Zhang},
      year={2025},
      eprint={2502.14361},
      archivePrefix={arXiv},
      primaryClass={cs.AI},
      url={https://arxiv.org/abs/2502.14361}, 
}

@misc{uesato2022solvingmathwordproblems,
      title={Solving math word problems with process- and outcome-based feedback}, 
      author={Jonathan Uesato and Nate Kushman and Ramana Kumar and Francis Song and Noah Siegel and Lisa Wang and Antonia Creswell and Geoffrey Irving and Irina Higgins},
      year={2022},
      eprint={2211.14275},
      archivePrefix={arXiv},
      primaryClass={cs.LG},
      url={https://arxiv.org/abs/2211.14275}, 
}

@misc{numina_math_datasets,
  author = {Jia LI and Edward Beeching and Lewis Tunstall and Ben Lipkin and Roman Soletskyi and Shengyi Costa Huang and Kashif Rasul and Longhui Yu and Albert Jiang and Ziju Shen and Zihan Qin and Bin Dong and Li Zhou and Yann Fleureau and Guillaume Lample and Stanislas Polu},
  title = {NuminaMath},
  year = {2024},
  publisher = {Numina},
  journal = {Hugging Face repository},
  howpublished = {\url{[https://huggingface.co/AI-MO/NuminaMath-CoT](https://github.com/project-numina/aimo-progress-prize/blob/main/report/numina_dataset.pdf)}}
}

@article{he2024olympiadbench,
  title={Olympiadbench: A challenging benchmark for promoting agi with olympiad-level bilingual multimodal scientific problems},
  author={He, Chaoqun and Luo, Renjie and Bai, Yuzhuo and Hu, Shengding and Thai, Zhen Leng and Shen, Junhao and Hu, Jinyi and Han, Xu and Huang, Yujie and Zhang, Yuxiang and others},
  journal={arXiv preprint arXiv:2402.14008},
  year={2024}
}

@article{grattafiori2024llama,
  title={The llama 3 herd of models},
  author={Grattafiori, Aaron and Dubey, Abhimanyu and Jauhri, Abhinav and Pandey, Abhinav and Kadian, Abhishek and Al-Dahle, Ahmad and Letman, Aiesha and Mathur, Akhil and Schelten, Alan and Vaughan, Alex and others},
  journal={arXiv preprint arXiv:2407.21783},
  year={2024}
}

@inproceedings{inproceedingsBoosting,
author = {Zhu, Lei},
year = {2023},
month = {11},
pages = {42-45},
title = {Boosting Image Classification Accuracy Leveraging Finer Grained Labels},
doi = {10.1109/ICICML60161.2023.10424766}
}

@article{articlehierarchical,
author = {Silla, Carlos and Freitas, Alex},
year = {2011},
month = {01},
pages = {31-72},
title = {A survey of hierarchical classification across different application domains},
volume = {22},
journal = {Data Mining and Knowledge Discovery},
doi = {10.1007/s10618-010-0175-9}
}

@misc{wang2025surveyordinalregressionapplications,
      title={A Survey on Ordinal Regression: Applications, Advances and Prospects}, 
      author={Jinhong Wang and Jintai Chen and Jian Liu and Dongqi Tang and Danny Z. Chen and Jian Wu},
      year={2025},
      eprint={2503.00952},
      archivePrefix={arXiv},
      primaryClass={cs.CV},
      url={https://arxiv.org/abs/2503.00952}, 
}

@misc{skywork_o1_open_series_2025,
  title        = {Skywork-o1 Open Series},
  author       = {{Skywork Team}},
  year         = {2024},
  howpublished = {\url{https://huggingface.co/Skywork}},
}

@article{yang2024qwen2,
  title={Qwen2 technical report},
  author={Yang, An and Yang, Baosong and Hui, Binyuan and Zheng, Bo and Yu, Bowen and Zhou, Chang and Li, Chengpeng and Li, Chengyuan and Liu, Dayiheng and Huang, Fei and others},
  journal={arXiv preprint arXiv:2407.10671},
  year={2024}
}

@article{hendrycks2021measuring,
  title={Measuring mathematical problem solving with the math dataset},
  author={Hendrycks, Dan and Burns, Collin and Kadavath, Saurav and Arora, Akul and Basart, Steven and Tang, Eric and Song, Dawn and Steinhardt, Jacob},
  journal={arXiv preprint arXiv:2103.03874},
  year={2021}
}

@article{pan2025specreason,
  title={Specreason: Fast and accurate inference-time compute via speculative reasoning},
  author={Pan, Rui and Dai, Yinwei and Zhang, Zhihao and Oliaro, Gabriele and Jia, Zhihao and Netravali, Ravi},
  journal={arXiv preprint arXiv:2504.07891},
  year={2025}
}

@article{khatri2025art,
  title={The Art of Scaling Reinforcement Learning Compute for LLMs},
  author={Khatri, Devvrit and Madaan, Lovish and Tiwari, Rishabh and Bansal, Rachit and Duvvuri, Sai Surya and Zaheer, Manzil and Dhillon, Inderjit S and Brandfonbrener, David and Agarwal, Rishabh},
  journal={arXiv preprint arXiv:2510.13786},
  year={2025}
}

@article{gholami2024ai,
  title={Ai and memory wall},
  author={Gholami, Amir and Yao, Zhewei and Kim, Sehoon and Hooper, Coleman and Mahoney, Michael W and Keutzer, Kurt},
  journal={IEEE Micro},
  volume={44},
  number={3},
  pages={33--39},
  year={2024},
  publisher={IEEE}
}

@article{uesato2022solving,
    title        = {Solving Math Word Problems With Process- and Outcome-Based Feedback},
    author       = {Uesato, Jonathan and Kushman, Nate and Kumar, Ramana and Song, Francis and Siegel, Noah and Wang, Lisa and Creswell, Antonia and Irving, Geoffrey and Higgins, Irina},
    year         = 2022,
    journal      = {arXiv preprint arXiv:2211.14275}
}
\appendix
\counterwithin{figure}{section}
\counterwithin{table}{section}
\appendix
\clearpage
\onecolumn

\section{Analyzing Compute Wastage in RSD}
\label{app:limitations}
\begin{figure*}[ht]
  \centering
  \begin{subfigure}[t]{0.7\textwidth}
    \centering
    \includegraphics[width=\linewidth]{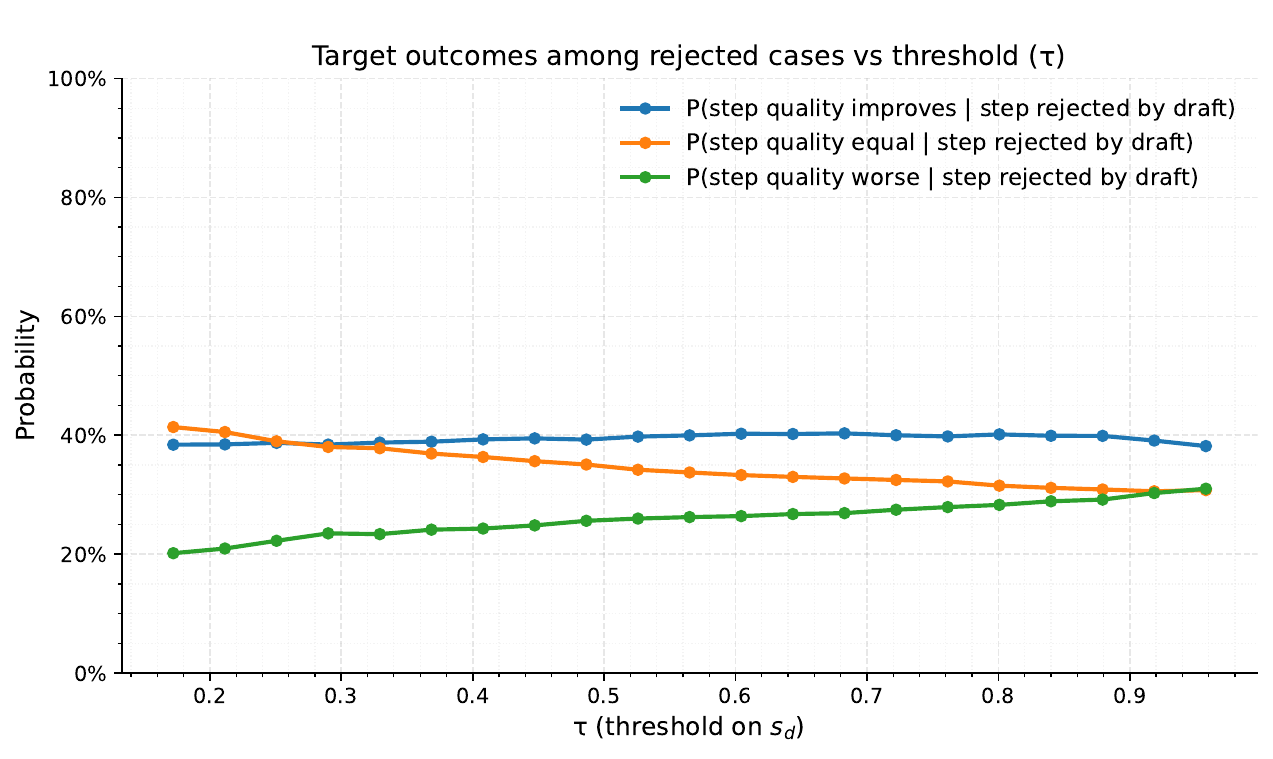}
    \caption{Outcomes among rejected cases vs.\ $\tau$}
    \label{fig:rsd_outcomes_vs_tau}
  \end{subfigure}
  \begin{subfigure}[t]{0.7\textwidth}
    \centering
    \includegraphics[width=\linewidth]{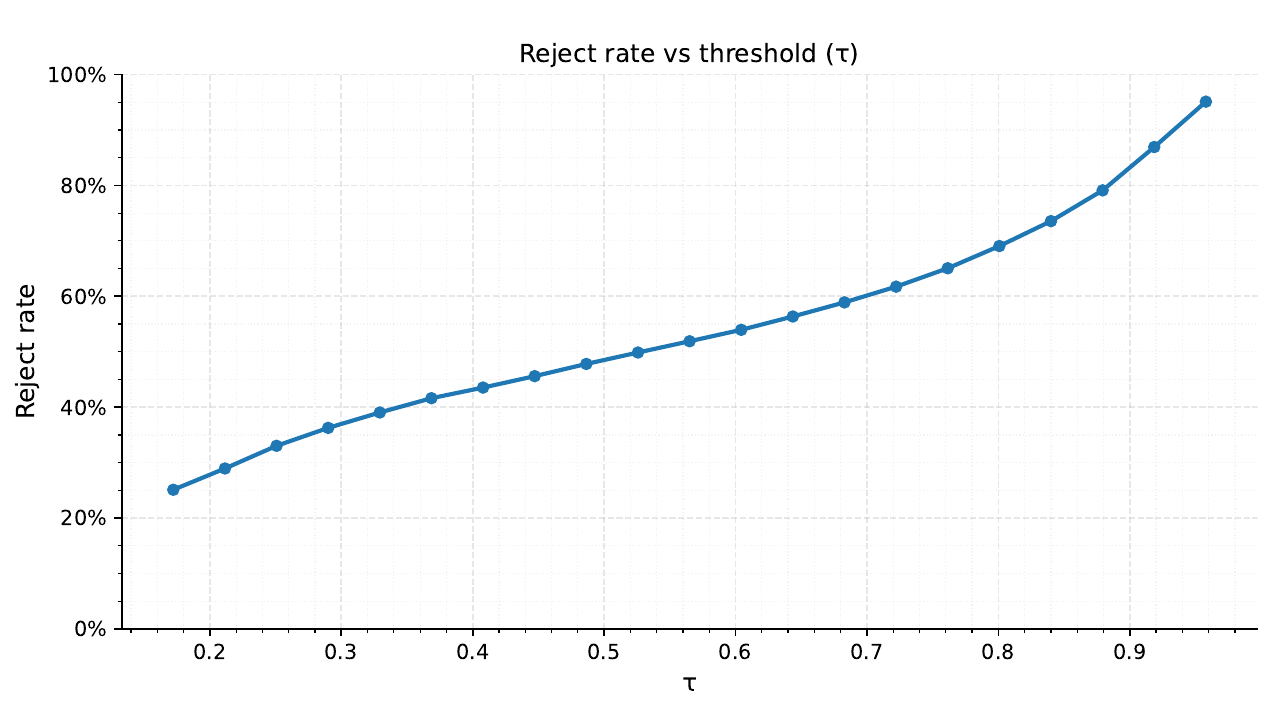}
    \caption{Reject rate vs.\ $\tau$}
    \label{fig:rsd_reject_rate_vs_tau}
  \end{subfigure}
  \begin{subfigure}[t]{0.7\textwidth}
    \centering
    \includegraphics[width=\linewidth]{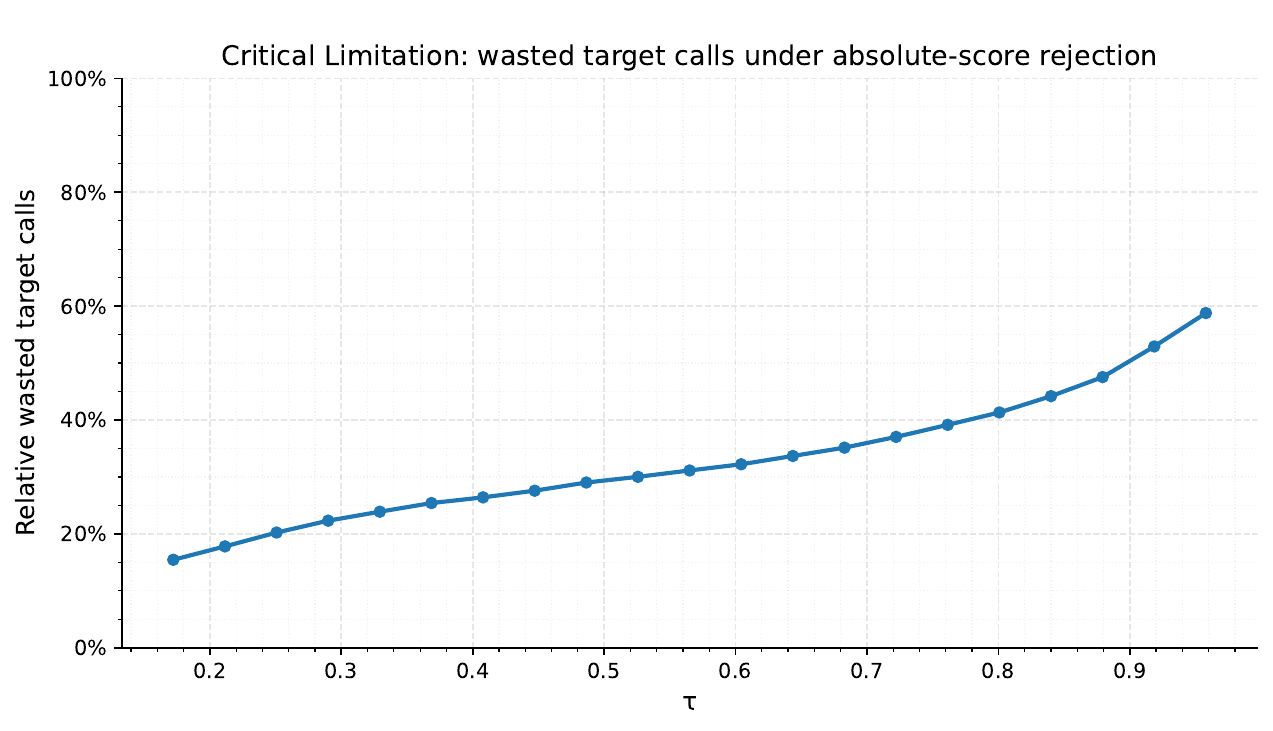}
    \caption{Wasted target calls vs.\ $\tau$}
    \label{fig:rsd_waste_proxy_vs_tau}
  \end{subfigure}
  \caption{Performance of the RSD method as a function of threshold $\tau$.}
  \label{fig:rsd_vs_tau}
\end{figure*}

Here we present a detailed analysis of RSD rejections; see Figure~\ref{fig:rsd_vs_tau}. The figure characterizes how RSD’s rejection behavior varies with the draft-score threshold~$\tau$ on~$s_d$. Among \emph{rejected} steps, the probability that the target improves the PRM score,
$\Pr(\text{step quality improves} \mid \text{step rejected by draft}, \tau)$, stabilizes at roughly~$0.4$ across a wide range of~$\tau$ (Figure~\ref{fig:rsd_outcomes_vs_tau}). This means that in approximately $60\%$ of target invocations there is \emph{no gain} (the step is equal or worse). At the same time, the reject rate $\Pr(s_d \le \tau)$ increases sharply with~$\tau$ (Figure~\ref{fig:rsd_reject_rate_vs_tau}), so absolute-score thresholding rapidly saturates the system with expensive target calls without increasing the likelihood of improvement.

We capture the resulting compute inefficiency via an explicit “wasted target calls’’ metric, derived from our wasted-computation formalization:
\[
\text{reject\_rate}(\tau)\cdot\bigl(1 - \Pr(\text{step quality improves} \mid \text{step rejected by draft}, \tau)\bigr),
\]
which grows monotonically with~$\tau$ (Figure~\ref{fig:rsd_waste_proxy_vs_tau}). Taken together, these trends indicate that when the draft fails due to \emph{shared} error modes with the target (e.g., hard reasoning instances), regenerating with the target often cannot salvage the step. Thus, RSD’s absolute-score rejection pays full target cost while delivering little to no quality benefit, motivating routing policies that depend on the \emph{relative} advantage $(s_t - s_d)$ rather than thresholding $s_d$ in isolation.

\section{Optimality of \OURSORACLE}
\label{sec:proof}
\begin{boxtheorem}[\OURSORACLE optimality under Budgeted Escalation]
\label{thm:oracle}
Given a budget $B\in\{0,\dots,N\}$ on the
number of escalations, consider policies $a=(a_1,\dots,a_N)\in\{0,1\}^N$ with $\sum_{i=1}^N a_i\le B$, where $a_i=1$ means ``escalate to target'' and $a_i=0$ means ``accept draft.'' Then there exists a threshold $\tau\in\mathbb{R}$ such that the policy $a^\star_\tau$ with $a^\star_{\tau,i}=\mathbb{I}\{\Delta_i>\tau\}$ attains
\[
\max_{a:\,\sum a_i\le B}\ \mathcal{S}(a)=\max_{a:\,\sum a_i\le B}\ \sum_{i=1}^N\bigl(s_{d,i}+a_i\Delta_i\bigr).
\]
Because $\sum_i s_{d,i}$ is independent of $a$, this problem reduces to
\[
\max_{a\in\{0,1\}^N}\ \sum_{i=1}^N a_i\Delta_i
\quad\text{s.t.}\quad \sum_{i=1}^N a_i\le B.
\tag{$\star$}
\]
\end{boxtheorem}
\begin{proof}
Suppose $\Delta_j<0$ and $a_j=1$; then setting $a_j\leftarrow 0$
strictly increases the objective while preserving feasibility, since the constraint is an
upper bound. Thus some optimal solution has $a_j=0$ for all $j$ with $\Delta_j<0$. Now, among indices with $\Delta_i>0$, let $a$ be any feasible
solution that includes $j$ but excludes $k$ with $\Delta_k>\Delta_j>0$. Swapping
$a_j\leftarrow 0$ and $a_k\leftarrow 1$ increases the objective by $\Delta_k-\Delta_j>0$ and
keeps feasibility, contradicting optimality. Repeating exchanges yields an optimal solution
that selects the $k=\min\{B,\ \#\{i:\Delta_i>0\}\}$ largest positive $\Delta_i$'s. 

In the constrained case, order $\Delta_{(1)}\ge\cdots\ge\Delta_{(N)}$ and choose
$\tau$ with $\Delta_{(k)}>\tau\ge\Delta_{(k+1)}$ (using $\Delta_{(0)}:=+\infty$,
$\Delta_{(N+1)}:=-\infty$). Then $a^\star_{\tau,i}=\mathbb{I}\{\Delta_i>\tau\}$ selects exactly
those $k$ indices; ties at $\Delta_i=\tau$ can be broken arbitrarily to satisfy the budget.
This achieves the maximum of $(\star)$, hence of $\mathcal{S}(a)$.
\end{proof}
\begin{corollary}[Pareto Frontier]\label{cor:pareto}
Let $\mathcal{C}(a)\in\mathbb{R}_+$ denote computational cost (e.g., \# target calls $=\sum_i a_i$).
As $\tau$ varies, the threshold policies $a^\star_\tau$ trace a cost–quality curve $\bigl(\mathcal{C}(a^\star_\tau),\,\mathcal{S}(a^\star_\tau)\bigr)$ that \emph{majorizes} any other policy's cost–quality pair: for any target cost $c$,
\[
\mathcal{S}(a^\star_{\tau(c)}) \;=\; \sup_{\pi:\,\mathcal{C}(\pi)\le c}\, \mathcal{S}(\pi),
\]
i.e., the oracle threshold family achieves the Pareto frontier of cost versus quality.
\end{corollary}
\begin{proof}
Fix any cost budget $c$ corresponding to some integer $B$. By Theorem~\ref{thm:oracle}, the optimal value among policies with $\sum_i a_i\le B$ is attained by $a^\star_\tau$ for a threshold $\tau$ that selects the $B$ largest $\Delta_i$'s. Thus no policy of cost at most $c$ can exceed $\mathcal{S}(a^\star_\tau)$, establishing the frontier property.
\end{proof}
\begin{corollary}[Zero Wasted Computation under Unconstrained Oracle]\label{cor:zero-waste}
In the unconstrained case, the oracle uses $a^\star_i=\mathbb{I}\{\Delta_i>0\}$.
Then the expected wasted computation, target calls that fail to improve quality, is
\[
\mathbb{E}\bigl[\mathbf{1}\{a^\star_i=1\}\,\mathbf{1}\{s_{t,i}\le s_{d,i}\}\bigr] \;=\; 0.
\]
\end{corollary}
\begin{proof}
By definition, $a^\star_i=1$ only if $\Delta_i>0$, which implies $s_{t,i}>s_{d,i}$. Hence $\mathbf{1}\{a^\star_i=1\}\,\mathbf{1}\{s_{t,i}\le s_{d,i}\}=0$ and the expectation is zero.
\end{proof}

\begin{remark}[\OURSORACLE Assumption]
The theorem and corollaries optimize the sum of step-local PRM scores, assuming a
decision at step $i$ does not affect future steps' scores.
\end{remark}
\section{Implementation Details}
\label{app:hparams}
Here, we include the training hyperparameters for the router; see Table \ref{tab:hyperparams}.
\begin{table}[hbt!]
\centering
\caption{Experimental hyperparameters.}
\label{tab:hyperparams}
\begin{tabular}{lc}
\toprule
Hyperparameter & Value\\
\midrule
Batch size & 1024 \\
Learning rate & 1e-5  \\
Epochs & 3 \\
\bottomrule
\end{tabular}
\end{table}
\section{\OURS Speculative Generation}
\label{app:algorithm}
The algorithm \ref{alg:arbitrage} details the inference procedure for \OURS, our step-level speculative generation framework. Given a question \(q\), the algorithm iteratively constructs a reasoning trace by alternating between draft generation and dynamic routing decisions. At each iteration, the draft model \(q_{\theta_d}\) auto-regressively generates a candidate reasoning step \(\vz_d\) until it emits a step separator token \(\langle\mathrm{sep}\rangle\) (e.g., \texttt{\textbackslash n\textbackslash n}) or the end-of-sequence token \(\langle\mathrm{eos}\rangle\). This step is then fed, along with the current context \(\vx\), into the lightweight \OURSROUTER \(h_{\theta_{\mathrm{router}}}\), which outputs a scalar \(\hat{y}\) estimating the likelihood that the target model produces a significantly better step. The router’s prediction is compared with a tunable threshold \(\tau\): if \(\hat{y} \leq \tau\), the draft step \(\vz_d\) is deemed sufficient and attached to the context; otherwise, the system falls back to the target model \(q_{\theta_t}\), which regenerates the step from the same prefix \(\vx\), yielding \(\vz_t\). \OURS minimizes redundant computation while preserving (or even enhancing) output correctness. The entire process is fully compatible with standard auto-regressive decoding pipelines and introduces minimal overhead due to the lightweight architecture of the router.

\begin{algorithm}[ht]
\caption{\OURS Inference}
\label{alg:arbitrage}
\DontPrintSemicolon

\KwIn{question $q$; draft model $q_{\theta_d}$; target model $q_{\theta_t}$; 
router $h_{\theta_{\mathrm{router}}}$; threshold $\tau$; step separator 
$\langle\mathrm{sep}\rangle$; end-of-sequence token $\langle\mathrm{eos}\rangle$; 
maximum number of steps $N$}
\KwOut{final decoded sequence $\vx$}

$\vx \gets q$\; \tcp*[r]{Initialize context with the question}

\For{$t \gets 1$ \KwTo $N$}{
  \If{last token of $\vx$ is $\langle\mathrm{eos}\rangle$}{
    \textbf{break}\; \tcp*[r]{finished early}
  }

  \tcp{Draft step}
  Sample draft step $\vz_d \sim q_{\theta_d}(\cdot \mid \vx)$ until the first of
  $\langle\mathrm{sep}\rangle$ or $\langle\mathrm{eos}\rangle$\;

  \tcp{Routing decision}
  $\hat{y} \gets h_{\theta_{\mathrm{router}}}(\vx, \vz_d)$\;

  \If{$\hat{y} \le \tau$}{
    $\vz^\star \gets \vz_d$\; \tcp*[r]{accept draft}
  }
  \Else{
    Sample target step $\vz_t \sim q_{\theta_t}(\cdot \mid \vx)$ until the first of
    $\langle\mathrm{sep}\rangle$ or $\langle\mathrm{eos}\rangle$\;
    $\vz^\star \gets \vz_t$\; \tcp*[r]{escalate to target}
  }

  \tcp{Append chosen step}
  $\vx \gets \vx \,\Vert\, \vz^\star$\;

  \If{last token of $\vx$ is $\langle\mathrm{eos}\rangle$}{
    \textbf{break}\;
  }
}

\Return $\vx$\;
\end{algorithm}

\section{Additional Results}
\label{app:additional_results}
We provide additional evaluation results that complement
Section~\ref{sec:algo_results}; see Figure~\ref{fig:additional_results}.

\begin{figure*}[t]
  \centering
  \setlength{\tabcolsep}{2pt}
  \begin{tabular}{cc}
    \includegraphics[width=0.48\textwidth]{figures/llama_family_acc_acp/llama_small_math.pdf} &
    \includegraphics[width=0.48\textwidth]{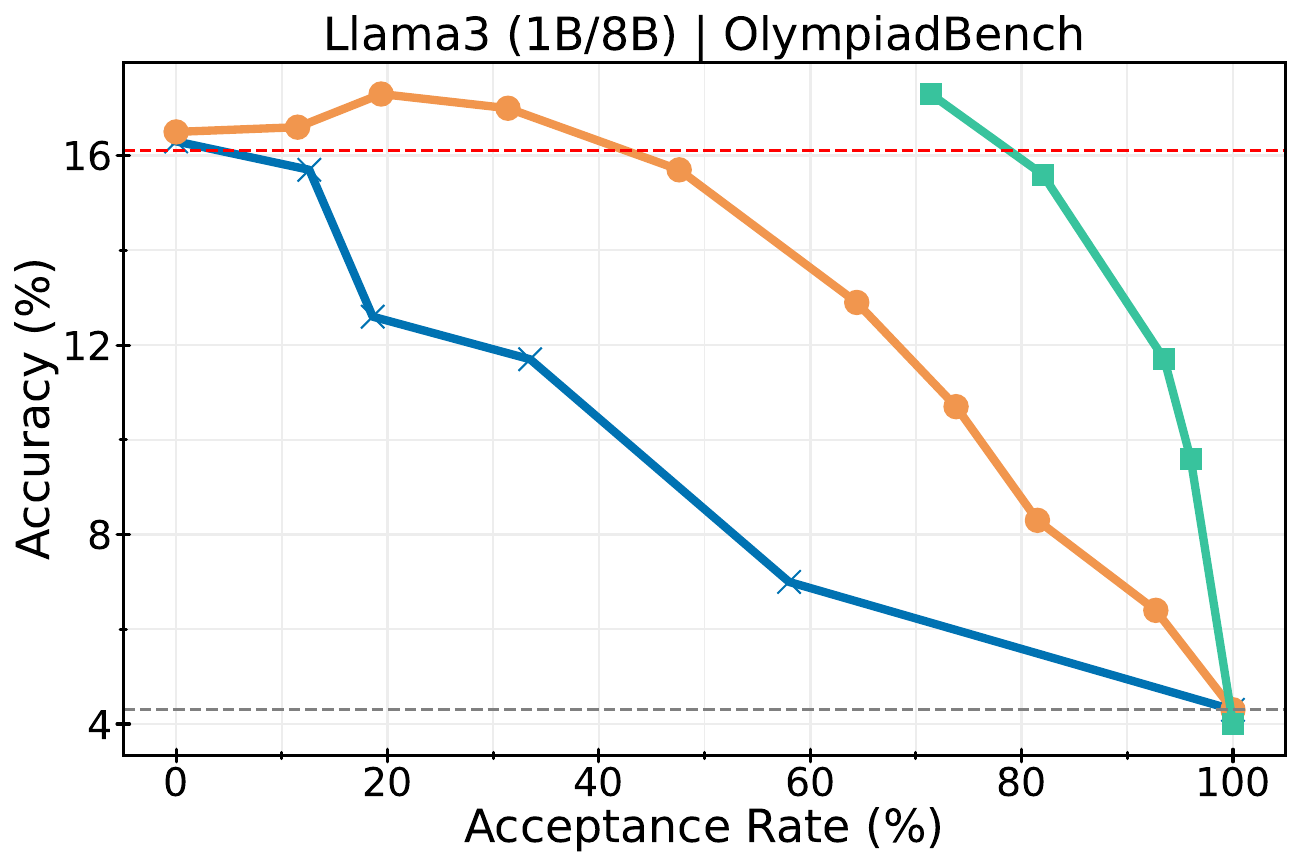} \\[2pt]
    \includegraphics[width=0.48\textwidth]{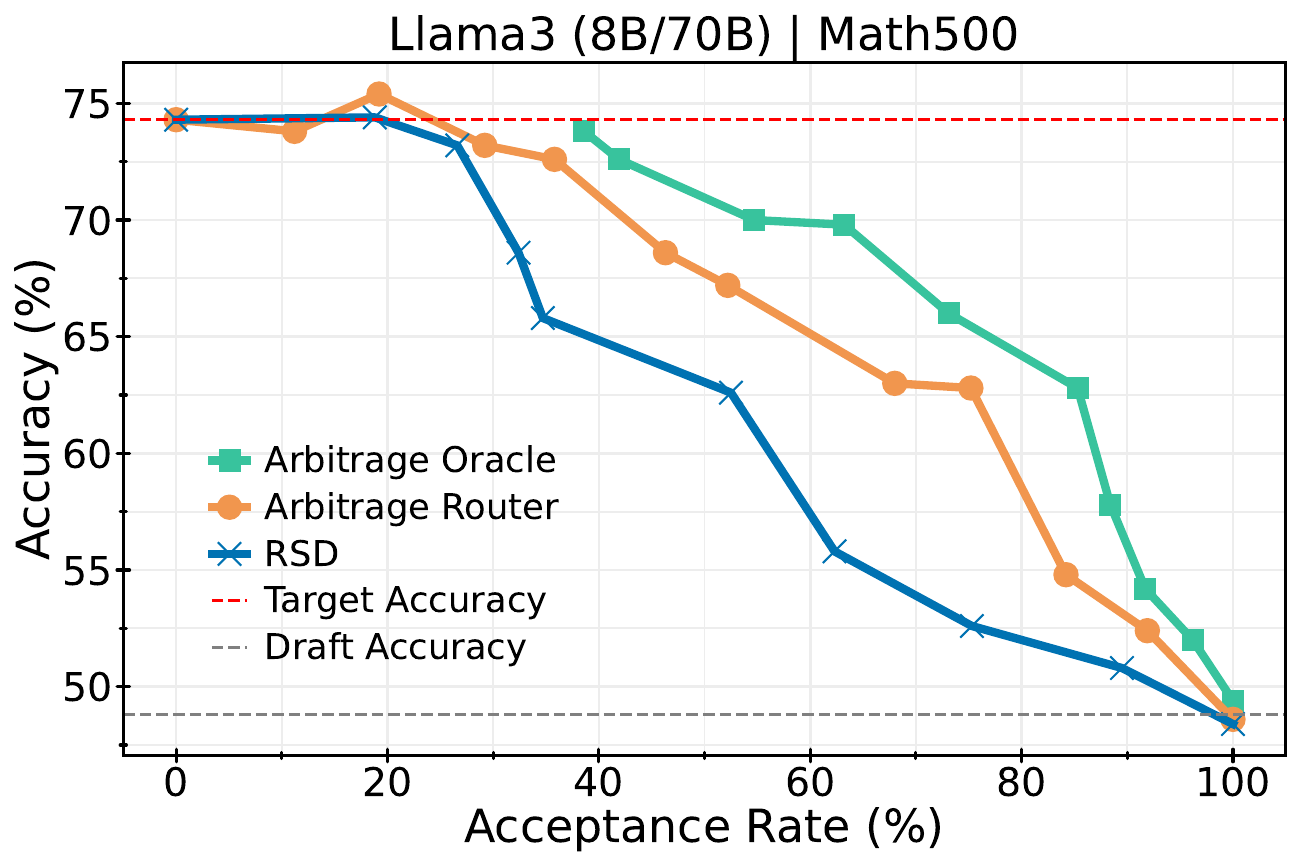} &
    \includegraphics[width=0.48\textwidth]{figures/llama_family_acc_acp/llama70b_olympiad.pdf} \\[2pt]
    \includegraphics[width=0.48\textwidth]{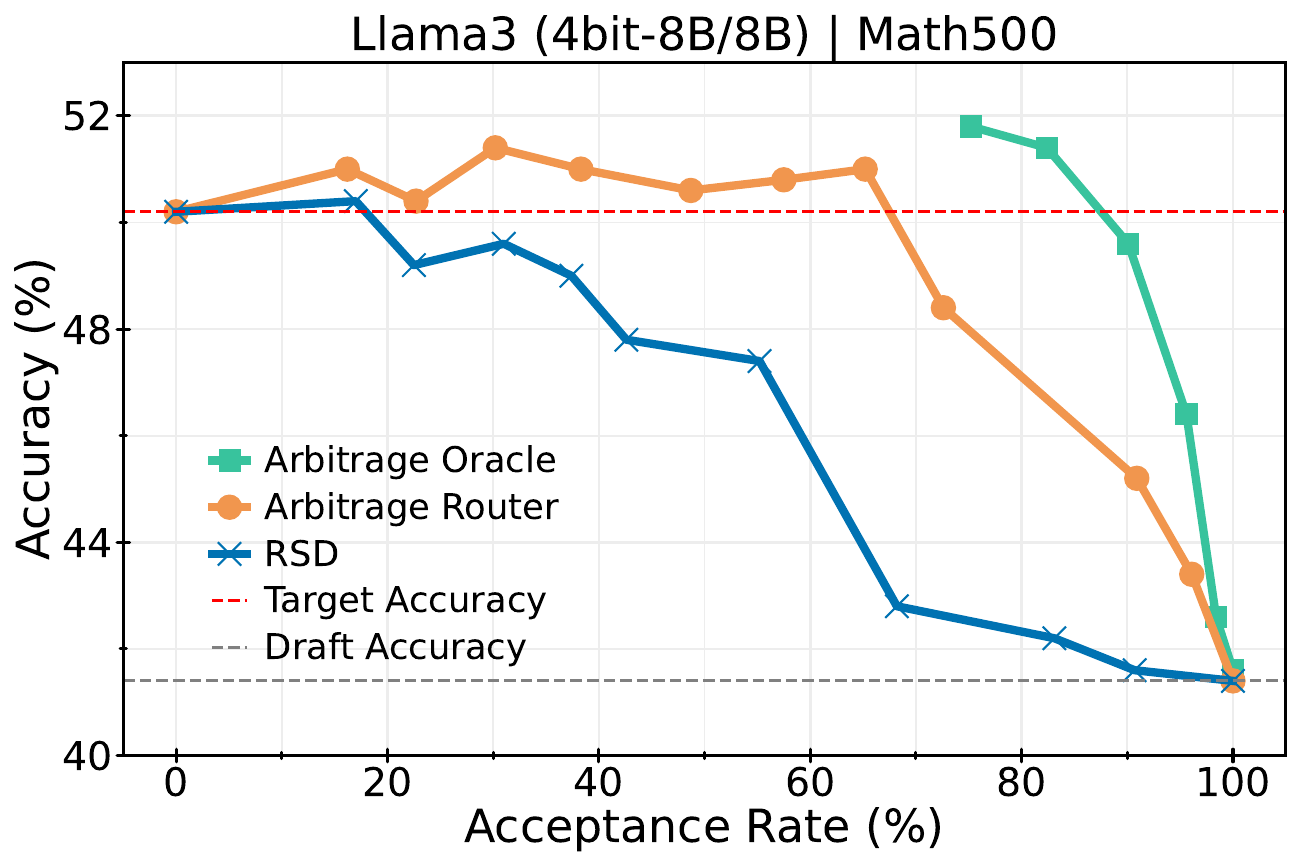} &
    \includegraphics[width=0.48\textwidth]{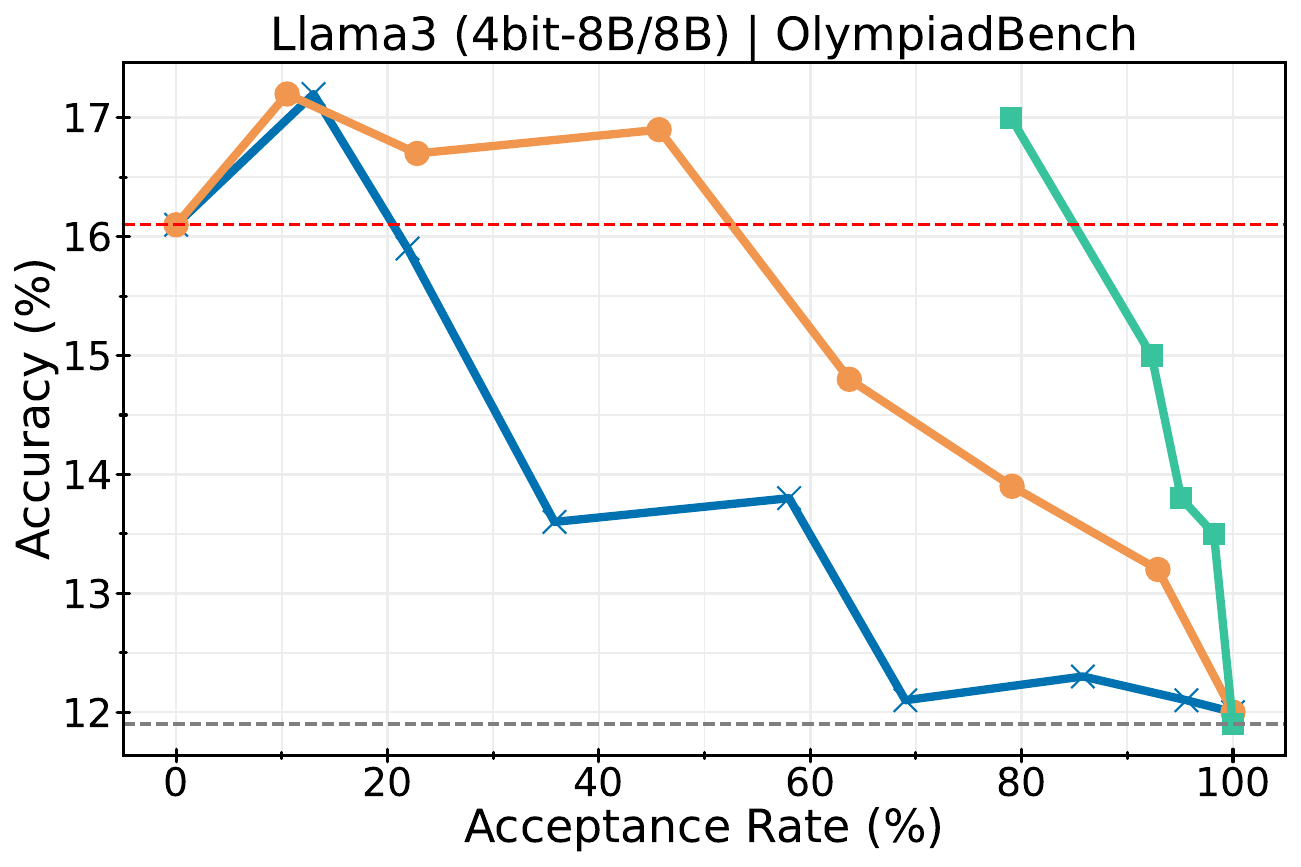} \\[2pt]
    \includegraphics[width=0.48\textwidth]{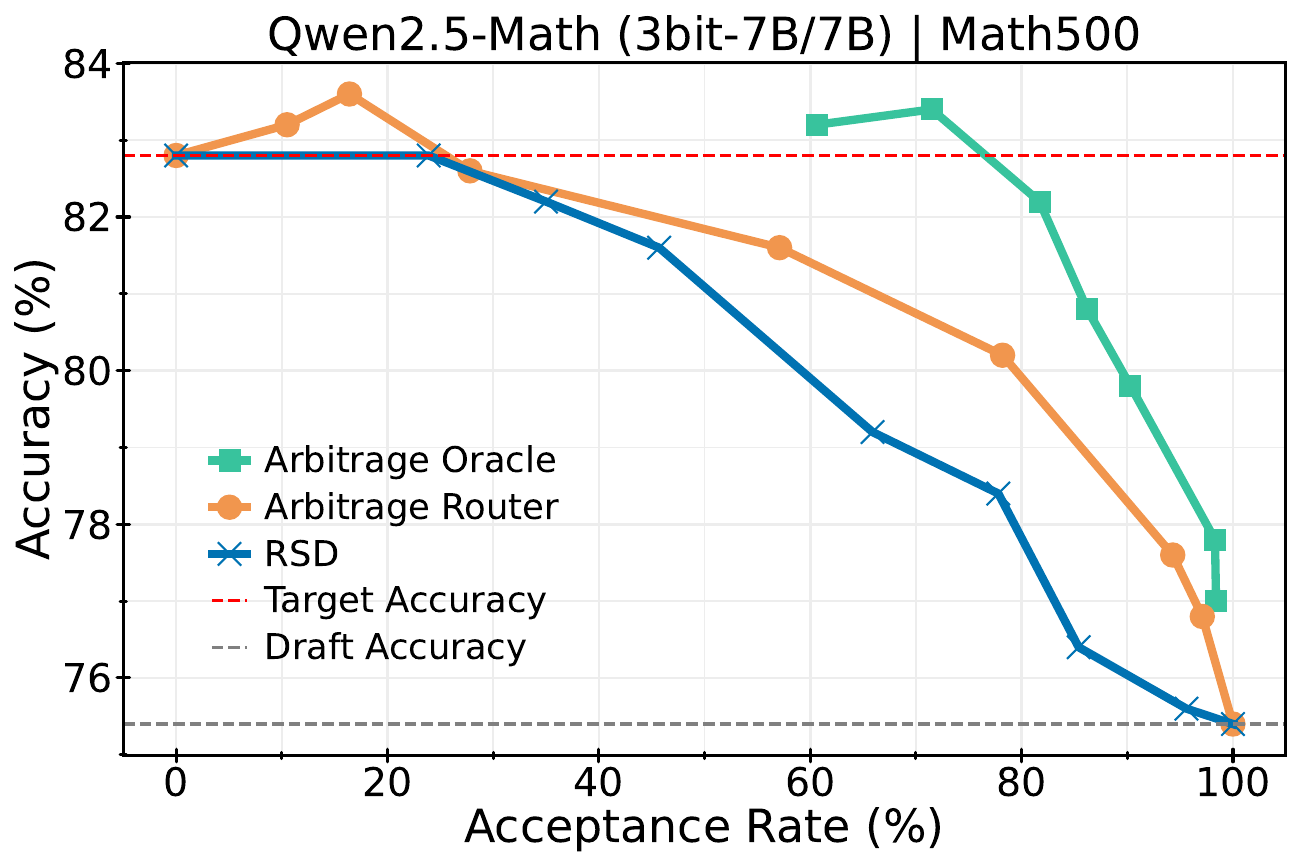} &
    \includegraphics[width=0.48\textwidth]{figures/qwen_family_acc_acp/qwen_q_olympiad.pdf}
  \end{tabular}
  \caption{
    \textbf{Additional compute--quality results.}
    Accuracy versus acceptance rate for \OURSORACLE, \OURSROUTER, and RSD on
    Math500 (left column) and OlympiadBench (right column) across the LLaMA3
    and Qwen2.5-Math model families. In all settings, \OURS achieves higher
    accuracy for a given acceptance rate, indicating a more favorable
    compute--quality trade-off.%
  }
  \label{fig:additional_results}
\end{figure*}

\end{document}